\documentclass{article}

\usepackage[main, final]{neurips_2026}

\usepackage[utf8]{inputenc} 
\usepackage[T1]{fontenc}    
\usepackage{hyperref}       
\usepackage{url}            
\usepackage{booktabs}       
\usepackage{amsfonts}       
\usepackage{nicefrac}       
\usepackage{microtype}      
\usepackage{xcolor}         

\usepackage{amssymb}            
\usepackage{mathtools}          
\usepackage{mathrsfs}           
\usepackage{graphicx}           
\usepackage{subcaption}         
\usepackage[space]{grffile}     
\usepackage{url}                
\usepackage{lipsum}             
\usepackage{comment}

\usepackage{bbm}

\newcommand{\spanv}[1]{\textrm{sp}(#1)}
\newcommand{\adv}[2]{\textrm{adv}(#1,#2)}
\DeclareMathOperator*{\argmax}{arg\,max}
\DeclareMathOperator{\E}{\mathbb{E}}
\DeclareMathOperator{\st}{\mathrm{st}}

\usepackage{amsthm}
\theoremstyle{plain}
\newtheorem{theorem}{Theorem}[section]

\newtheorem{lemma}[theorem]{Lemma}
\newtheorem{corollary}[theorem]{Corollary}
\theoremstyle{definition}
\newtheorem{definition}[theorem]{Definition}
\newtheorem{assumption}[theorem]{Assumption}
\theoremstyle{remark}

\title{Revisiting Value Iteration: Unified Analysis of Discounted and Average-Reward Cases}

\author{
Arsenii Mustafin\\
Cyber-physical Systems Group\\
Aalto University\\
Espoo, Finland 02150\\
\texttt{arsenii.mustafin@aalto.fi}\\
\And
Xinyi Sheng\\
Cyber-physical Systems Group\\
Aalto University\\
Espoo, Finland 02150\\
\texttt{xinyi.sheng@aalto.fi}\\
\And
Dominik Baumann\\
Cyber-physical Systems Group\\
Aalto University\\
Espoo, Finland 02150\\
\texttt{dominik.baumann@aalto.fi}\\
}

\begin{document}

\maketitle

\begin{abstract}
While Value Iteration (VI) is one of the most fundamental algorithms in 
Reinforcement Learning, its theoretical convergence guarantees still exhibit 
a persistent mismatch with empirical behavior. In the discounted-reward case, 
classical theory guarantees geometric convergence with rate $\gamma$, while 
in the average-reward case recent work suggests that only sublinear 
convergence can be expected. In practice, however, VI is often observed to 
converge significantly faster. In this work, we show through a unified 
geometry-based analysis that, under an assumption of a unique and unichain 
optimal policy, \emph{(i)} convergence is geometric in \emph{both} the 
discounted- and average-reward settings and \emph{(ii)} the convergence rate 
is faster than previous analyses suggest.
\end{abstract}

\section{Introduction} \label{sec:intro}

Value Iteration (VI), introduced by Bellman in the late 1950s, is one of the oldest and most fundamental algorithms for solving Markov Decision Processes (MDPs).
However, as we argue in this paper, its convergence properties are still not fully understood. 
\citet{howard1960dynamic} showed that, in discounted reward MDPs, VI converges geometrically in the $\ell^\infty$ norm with rate equal to the discount factor $\gamma$, that this rate is tight in the worst case, and that, as $\gamma$ approaches 1, worst-case bounds become sublinear.
The understanding of VI's convergence is still largely based on this early result. 
Moreover, \citet{lee2025optimal} suggest that in the average-reward setting with $\gamma=1$, sublinear convergence is not only a worst-case bound, but also optimal.
Thus, the case appears to be closed.
Nevertheless, in this paper, we show that this impression is incomplete and potentially misleading.

Let us consider an instructive example, inspired by examples commonly used in RL theory lectures (e.g., by \cite{RLTheoryLec3}).
We have an MDP with three groups of states, ``heaven,'' ``purgatory,'' and ``hell.''
Each group itself consists of three states.
Intuitively, heaven is the goal, forming a rewarding cycle under the optimal policy, with small probabilities of slipping into purgatory or hell. 
The full specification is given in Appendix~\ref{app:mdp_example}.

Suppose we use VI to solve this MDP.
Given the discussion above, we should expect geometric convergence with rate $\gamma$ in the discounted reward case and sublinear convergence as $\gamma$ approaches~1.
Nevertheless, when running VI on this MDP and plotting the span seminorm of the error vector (normalized by its initial value) for different $\gamma$, Figure~\ref{fig:VI_conv_example}(A) shows that, even as $\gamma$ approaches 1, the convergence rate does not become sublinear.
In fact, we still see geometric convergence with a rate strictly smaller than 1.
This result might, of course, be an artifact of a carefully handcrafted MDP.
Thus, Figure~\ref{fig:VI_conv_example}(B) shows the convergence behavior of VI on a random MDP, which exhibits the same characteristics.
Existing literature cannot explain these empirical results.

In this paper, we fill this gap through a unified analysis of the discounted and average-reward settings, which are typically treated separately, yielding new convergence rates for both.
This has direct practical relevance. 
For instance, in modern reinforcement learning (RL), VI underlies the critic update in actor--critic methods, often implemented with neural networks as function approximators. 
When practitioners observe slow convergence, it is unclear whether this stems from approximation error, optimization issues, or the fundamental convergence behavior of VI itself. 
Sharp theoretical guarantees, as we provide in this paper, enable us to disentangle these sources of slowness.

\paragraph{Contributions.}
In this paper, we provide a new analysis of the VI algorithm convergence that unifies the discounted and average-reward settings. Our main contributions are as follows:
\begin{itemize}
    \item We show that, in the discounted-reward case, under the assumption that the unique optimal policy is unichain, VI exhibits geometric convergence with rate $\iota^{1/n^2} \gamma < \gamma$ in the span seminorm. Thus, the total number of iterations required to obtain an $\epsilon$-optimal policy is 
    $$ \mathcal{O}\left( \frac{\log(1/\epsilon) + \log(1/(1-\gamma))}{\log(1/\gamma) + \frac{\log(1/\iota)}{n^2} } \right). $$
    \item We show that in the average-reward case, under the same assumption, the VI algorithm exhibits geometric convergence with a rate $\iota $ in terms of the span seminorm, which implies that the total number of iterations required to obtain an $\epsilon$-optimal policy is 
    $$\mathcal{O}\left( \frac{\log(1/\epsilon)}{\frac{\log(1/\iota)}{n^2} } \right).$$
    \item We expand a geometric interpretation of MDPs \citep{mdp_geometry} previously introduced for the discounted reward case onto the average-reward case. We show that in both cases, VI has the same dynamics, allowing us to analyze both cases together.
\end{itemize}

\begin{figure*}[!t] 
\begin{center}
\includegraphics[width=\textwidth]{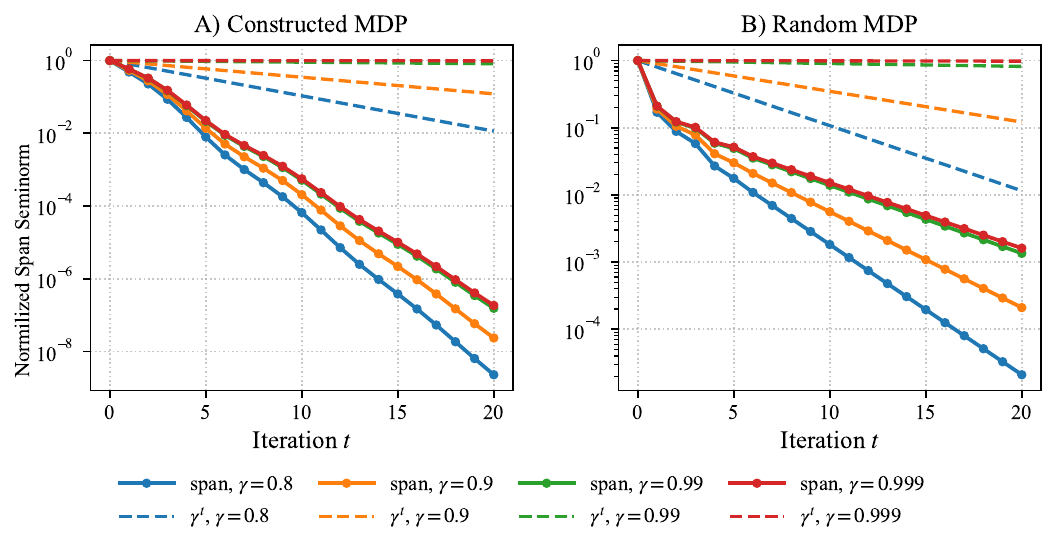}
\end{center}
\caption{Normalized span seminorm versus iteration $t$ (log scale) for four discount factors $\gamma\in\{0.8,0.9,0.99,0.999\}$. 
Subfigure (A) shows an MDP used as an example in Section \ref{sec:intro}; a detailed description is given in Appendix \ref{app:mdp_example}. Subfigure (B) shows a random MDP with 10 states, 1 to 4 actions on each state, and random rewards. 
Solid curves are the empirical spans (each normalized by its initial value). 
Across both MDPs, as $\gamma$ approaches 1, the empirical spans decay rate approaches a geometrical convergence with a rate significantly faster than 1.} 
\label{fig:VI_conv_example}
\end{figure*}

\paragraph{Previous Work.} MDPs were introduced in the late 1950s, together with foundational algorithms such as VI \citep{bellmandp} and a comprehensive dynamic programming framework \citep{howard1960dynamic}. Two standard performance criteria emerged: the discounted reward criterion, extensively summarized by \citet{Puterman1994}, and the average-reward criterion, developed through the works of \citet{blackwell1962}, \citet{yushkevich1974class}, \citet{denardo1968multichain}, and characterized further by \citet{schweitzer1978functional,schweitzer1984existence}.

For discounted MDPs, \citet{howard1960dynamic} showed that the convergence rate of VI is upper-bounded by the discount factor $\gamma$, and that this bound is attainable. Subsequent works refined and systematized these results and their assumptions \citep{Puterman1994,feinberg2014value, feinberg2020complexity}, establishing VI as a central tool for computing value functions in the discounted case. More recently \cite{leeRyuAnchoring2023} show a sublinear convergence with a rate $1/t$.

For average-reward MDPs, convergence of VI has been extensively studied. In the unichain case, linear convergence rates can be obtained via the $\delta$- and ergodicity coefficients \citep[e.g.,][]{seneta2006nonnegative,huebner1977,puterman2014}, and $J$-stage span contractions yield linear rates every $J$ iterations in the span seminorm \citep{federgruen1978,vanderwal1981,puterman2014}. In the multichain case, normalized VI iterates converge to the optimal average reward, although the policy error may not vanish \citep[Thm.~9.4.1]{puterman2014}. Necessary and sufficient conditions for VI convergence, as well as asymptotic linear rates on the Bellman error, were established by \citet{schweitzer1977,schweitzer1979}. To ensure convergence of iterates to the bias function $h^\star$, \citet{white1963} introduced Relative Value Iteration (RVI), with sufficient conditions studied by \citet{morton1977}, and additional conditions for VI convergence were studied by \citet{dellavecchia2012}.
We go beyond those results by showing geometric convergence of VI also in the average-reward setting with respect to the span seminorm.

Several works have also investigated the relationship between discounted MDPs and average-reward MDPs. In particular, \citet{Puterman1994} reveals a direct connection between discounted and average-reward quantities. 
Connections at the algorithmic level were studied by \citet{tsitsiklis2002average}. 
From an optimality perspective, discounted, average, and Blackwell optimality were unified through sensitive discount optimality \citep{mahadevan1996sensitive}, while \citet{grand2023reducing} showed that average optimality can be reduced to discounted optimality.
These results demonstrate deep connections between the discounted- and
average-reward settings, yet researchers typically treat the two cases
separately through different value representations.

\section{Preliminaries and Problem Statement}
\paragraph{MDP Model.} We employ a standard discrete finite MDP framework \citep{puterman2014}, modifying it slightly by redefining the SAP space for convenience. That is, we define the MDP as the tuple $\mathcal{M} = 
\langle \mathcal{S}, \mathcal{A}, \mathcal{P}, \mathcal{R}, \gamma \rangle$, where 
$\mathcal{S} = \{s_1, \ldots, s_n\}$ is a finite set of $n$ states, and $\mathcal{A}$ is the set of all possible actions. Following \citet{mdp_geometry}, we make the set of attainable state–--action pairs the main object of our analysis and define it as $\mathcal{Z}$; we will use this concept extensively and, for convenience, refer to a state–--action as a \textbf{SAP}. Since every SAP $z=(s,a)$ is tied to exactly one state, we define an auxiliary mapping $\mathrm{st}$ such that $\mathrm{st}(z) = s$ means that the state of SAP $z$ is $s$ ($z=(s,a')$ for some action $a'$). Each SAP $z$ is characterized by a transition probability vector $\mathcal{P}(z) = 
(p^z_1, \dots, p^z_n)$, where $p^z_i$ defines the probability of transitioning to state $s_i$ after selecting $z$, drawn from the set of all transition probabilities $\mathcal{P}$, and a deterministic reward $r^z \in [0,1]$ drawn from the set of all rewards $\mathcal{R}$ (assuming scaled rewards). Lastly, 
$\gamma \in (0,1]$ is a discount factor that specifies how much less a reward is valued if it is received one step later.

Given the set of SAPs, a choice rule (or \textbf{policy}) is a map $\pi: \mathcal{S} \rightarrow \mathcal{A}$.
We consider only deterministic stationary policies, which means that for each state $s$, the same action is chosen whenever the state is visited. Therefore, it is convenient to view a 
\textbf{policy as a set of SAPs}, $\pi = \{z_1, \dots, z_n\}$ with $z_i = (s_i, \pi(s_i))$. If a policy $\pi$ chooses SAP $z$, we write $z \in \pi$.

Once the policy is fixed, the MDP dynamics reduce to a Markov Chain. A chain is called \textbf{unichain} if it contains a single recurrent class, possibly with a non-empty set of transient states. Otherwise, the MDP is called \textbf{multichain} following \citet{puterman2014}.

For a given MDP $\mathcal{M}$ and policy $\pi$, performance can be measured in two ways corresponding to the discounted and average-reward criteria. We start with the former, as it is slightly simpler.

\paragraph{Discounted MDP.} In the discounted reward case, the value of the policy $\pi$ in state $s$, denoted as $V^\pi(s)$, is defined as the expected discounted reward obtained by starting from $s$ and following $\pi$,$
V^\pi(s) = \E \left[ \sum_{t=0}^{\infty} \gamma^t r_t \middle| s\right],$
where $r_t$ is the reward obtained at time $t$. The vector of values $V^\pi$ is the unique vector satisfying the Bellman equation $ V^\pi(s) = r^{\pi(s)} +  \gamma \sum_{s'}p_{s'}^{\pi(s)} V^\pi(s') $. We can also write this in matrix form as $V^\pi = R^{\pi} +  \gamma P^\pi V^\pi$, where $R^{\pi}$ is a vector of all rewards of policy $\pi$, so that the $i$th entry  $R^\pi(i) = r^{\pi(i)}$; and $P^\pi$ is a transition kernel of the Markov chain induced by policy $\pi$. It is a square matrix whose $i$th row is equal to $P^\pi(i)=\mathcal{P}(\pi(i))=(p^{\pi(i)}_1, \dots, p^{\pi(i)}_n)$.

In the discounted reward case, the optimal policy $\pi^*$ is a policy that attains the maximum values in every state: $V^{\pi^*} (s) \ge V^\pi(s)$ for any $ \pi$ and  $s$. We are also often interested in an approximate solution, or an $\epsilon$-optimal policy $\pi^\epsilon$ ($\epsilon >0$) such that $V^{\pi^*} (s) -  V^{\pi^\epsilon}(s) <\epsilon $.

\paragraph{Average-reward MDP.} In the average-reward case, we set $\gamma=1$, i.e., all rewards are equally important to the agent, independent of when they are received. This case is considered more challenging, as tackling infinite values involves two quantities per policy: the \emph{gain} and the \emph{bias}. 

For a given policy $\pi$, the \emph{average reward} (or \emph{gain}) from state $s$ is basically equivalent to the values in the discounted case with $\gamma=1$, but we average along time, expressed as $\rho^\pi(s) \;=\; \lim_{T\to\infty} \frac{1}{T}\,\E\!\left[ \sum_{t=0}^{T-1} r_t \,\middle|\, s \right]$. If the MDP is unichain, then the average reward is state-independent, i.e., $\rho^\pi(s)\equiv\rho^\pi$ for any start state $s\in\mathcal{S}$ \citep{mahadevan1996average}. The \emph{relative value function} (or \emph{bias}) from state $s$ measures the difference between the actual rewards under $\pi$ starting from $s$ and the gain, $h^\pi(s) \;=\; \E\!\left[ \sum_{t=0}^\infty \bigl( r_t - \rho^\pi \bigr) \,\middle|\, s \right].$

A policy $\pi^*$ is \emph{optimal} in the average-reward case if \emph{(i)} it has the highest possible gain, i.e., $\forall\,\pi,\,s:\ \rho^{\pi^*}(s)\ge \rho^\pi(s)$; and \emph{(ii)} it has the highest bias with respect to the optimal gain, i.e., $\forall\,\pi,\,s:\ h^{\pi^*}(s)\ge h^\pi(s)$. An $\epsilon$-optimal gain policy $\pi^\epsilon$ is a policy such that its 
gain at each state is at most $\epsilon$ worse than the optimal gain, i.e.,
$\rho^{\pi^*}(s) - \rho^{\pi^\epsilon}(s) \le \epsilon \,\, \forall s$.

\paragraph{Advantages and Value Iteration Algorithm.} 
A fundamental algorithm for solving MDPs, and the primary object of study in this paper, is the \textbf{Value Iteration} (VI) algorithm, which starts from an arbitrary value vector $V_0$ and then iteratively performs the update
\begin{equation} \label{eq:vi_update}
V_{t+1}(s) \;=\; \max_{\mathrm{st}(z)=s} \left\{\, r^z \;+\; \gamma \sum_{i=1}^n p^z_i\, V_t(s_i) \right\}.
\end{equation}

Alternatively, VI can be expressed in terms of \textbf{advantages} \citep{mdp_geometry}, thereby simplifying the subsequent analysis.
Thus, let us next introduce the concept of advantages.
The advantage of a SAP $z=(s,a)$ with respect to a policy $\pi$ is defined as the gain (or loss) caused by a one-time deviation from policy $\pi$ by choosing action $a$ at state $s$ instead of following $\pi$ starting from $s$.
In the discounted-reward case, advantages are always well-defined and given by
\begin{equation} \label{eq:def_advantage}
\adv{z}{\pi} \;=\; r^z \;+\; \gamma \sum_{i=1}^n p^z_i\, V^\pi(i) \;-\; V^\pi(s),
\end{equation}
and this quantity exists for all SAPs and for all policies.

In the average-reward case, we define the advantage function as \citep{zhang2020average,adamczyk2025average}: $ \adv{z}{\pi} = Q^\pi(z) - h^\pi(s) $, where 
$$Q^\pi(z)=\E \left[ \sum_{t=0}^\infty \left( r_t - \rho^\pi \right) \mid s_0=s,a_0=a \right] = r^z - \rho^\pi +\sum_{i=1}^n p^{z}_{i} h^{\pi}(s_i).$$ 
Then,
\begin{align*}
    \adv{z}{\pi} = r^z - \rho^\pi +\sum_{s'}p^{z}_{s'} h^{\pi}(s') - h^\pi(s).
\end{align*}

Note that in both cases, we do not need the values to be produced by an actual policy to compute an advantage or that we can compute advantage with respect to any set of values, which allows to express VI in terms of advantages:
\begin{equation}
V_{t+1}(s) \;=\; V_t(s) \;+\; \max_{\mathrm{st}(z)=s} \, \adv{z}{V_t},
\end{equation}
and we will adopt this perspective here, as it aligns naturally with the geometric view.

Finally, we analyze the convergence of VI with respect to the vector span seminorm, defined as the difference between the maximum and minimum entries of a vector. For a value vector $V$, we denote it by $\spanv{V} = \max_i V(s_i) - \min_j V(s_j)$.

\paragraph{Geometric interpretation of MDPs.} The discounted-reward MDP model 
presented above was reinterpreted in geometric terms by 
\citet{mdp_geometry}, and our analysis relies on this interpretation. In 
this framework, key objects of an MDP---SAPs and policies---are viewed 
as points and hyperplanes in a linear space called the \emph{SAP space}, which we make precise below. 
Consequently, the dynamics of VI can be interpreted as the movement of a 
hyperplane in this space.

In this paper, we extend this interpretation to the average-reward case. 
We show that, although the definition of values must be modified so that 
they still correctly define a hyperplane, the two settings are geometrically 
equivalent, which enables a unified analysis.

Before, we briefly introduce the geometric interpretation and define the 
quantities required for our analysis; we refer the reader to the original 
paper for a detailed exposition~\citep{mdp_geometry}.

The SAP space is an $(n+1)$-dimensional space, where $n$ is the number of 
states. One coordinate, referred to as the $0$-th coordinate or the 
\emph{height}, is special and corresponds to SAP rewards and policy values. 
To construct the \textbf{SAP vector} $z^+=(c_0,\dots,c_n)$ from a SAP $z$ 
on state $s$ ($\st(z)=s$), we set its first entry (the $0$-th coordinate) to 
the SAP reward: $c_0=r^z$. The next $n$ coordinates correspond to states and 
are defined as $c^z_i=\gamma p^z_i$ for $i\ne s$ and 
$c^z_s=\gamma p^z_s-1$. Hence, the sum of the last $n$ coordinates satisfies 
$\sum_{i=1}^n c^z_i=\gamma-1$, with the $s$-th entry negative and all others 
positive.

For a policy $\pi$, the \textbf{policy vector} 
$V^\pi_+=(1,V^\pi(s_1),\dots,V^\pi(s_n))^\top$ concatenates the state values with 
a $1$ in the $0$-th coordinate.
Geometrically, $\pi$ is represented by the 
hyperplane $\mathcal{H}^\pi$ consisting of all vectors orthogonal to 
$V^\pi_+$ (with respect to the standard Euclidean inner product), which is 
the span of the SAP vectors forming the policy. Such a hyperplane can be 
constructed for any value vector $V_+$, i.e., we do not require an explicit 
policy; for example, we can form a hyperplane for $V_t$ produced by $t$ 
iterations of VI. Hyperplanes arising from value vectors without associated 
policies are called \textbf{pseudo-policies}.

The key property of action and policy vectors is that for any SAP vector 
$z^+$ and policy vector $V^\pi_+$, the inner product $z^+ V^\pi_+$ equals the 
advantage of $z$ under $\pi$ and corresponds to the oriented vertical 
distance from $z^+$ to $\mathcal{H}^\pi$. Consequently, if SAP $z$ is in policy $\pi$, the advantage of $z$ with respect to $\pi$ is 0 and vectors $z^+$ and $V^\pi_+$ are orthogonal. An action-space example for a  two-state MDP is shown in Figure~\ref{fig:first}.

One tool provided by the geometric interpretation that we use here is the
equivalence transformation $\mathcal{L}^\delta_s$, which transforms an MDP
$\mathcal{M}$ into a new MDP $\tilde{\mathcal{M}}
=\mathcal{L}^\delta_s\mathcal{M}$. The transformation preserves transition
probabilities, but modifies rewards of SAPs such that for any policy $\pi$ the
original values $V^\pi$ in $\mathcal{M}$ and transformed values
$\tilde{V}^\pi$ in $\tilde{\mathcal{M}}$ satisfy:
\vspace{-5pt}
\begin{itemize}
\item $\tilde{V}^\pi(s)=V^\pi(s)+\delta$ for the transformed state $s$,
\item $\tilde{V}^\pi(s')=V^\pi(s')$ for any state $s'\neq s$,
\item the advantages $\adv{z}{\pi}$ of all SAPs remain unchanged.
\end{itemize}
\vspace{-5pt}
This transformation allows us to obtain a \textbf{normalization} by applying
$n$ transformations, one per state:
$\mathcal{L}^{-v^*(s_i)}_{s_i}$. The resulting MDP $\mathcal{M}^*$ has
optimal values equal to $0$ at every state, so the hyperplane
$\mathcal{H}^{\pi^*}$ corresponding to the optimal policy becomes horizontal
at level $0$. Consequently, the rewards of SAPs participating in optimal
policies are $0$, while the rewards of all other SAPs are negative. The
other properties of this normalization relevant to our analysis are
discussed in the next section.

\section{Unified Geometric Interpretation and New Policy Values} \label{sec:new geo interpr with lemmas}

\begin{figure}[t]
  \centering
  \begin{subfigure}{0.49\textwidth}
    \centering
    \includegraphics[width=\linewidth]{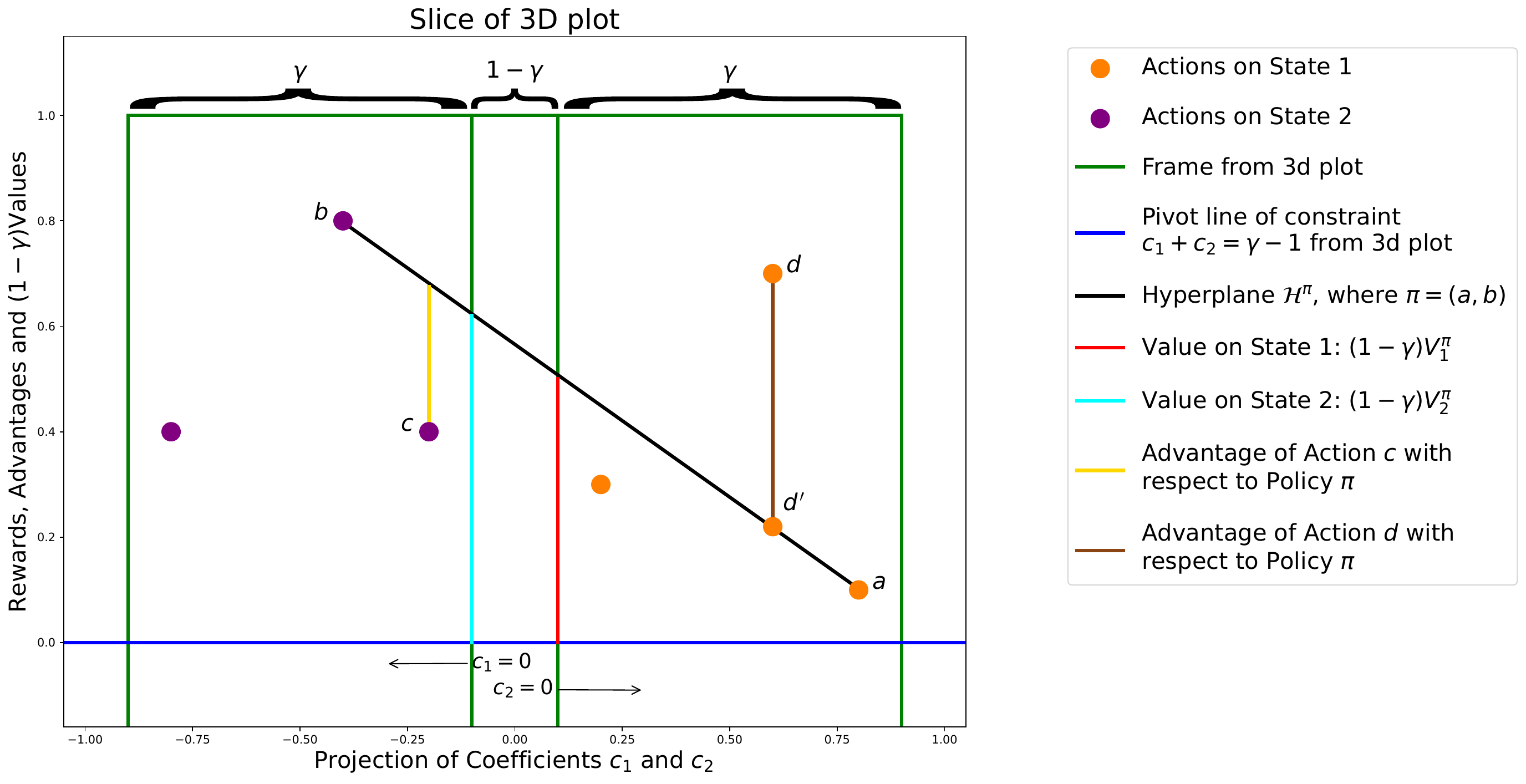}
    \caption{2-State MDP, old interpretation, discounted reward.}
    \label{fig:first}
  \end{subfigure}\hfill
  \begin{subfigure}{0.49\textwidth}
    \centering
    \includegraphics[width=\linewidth]{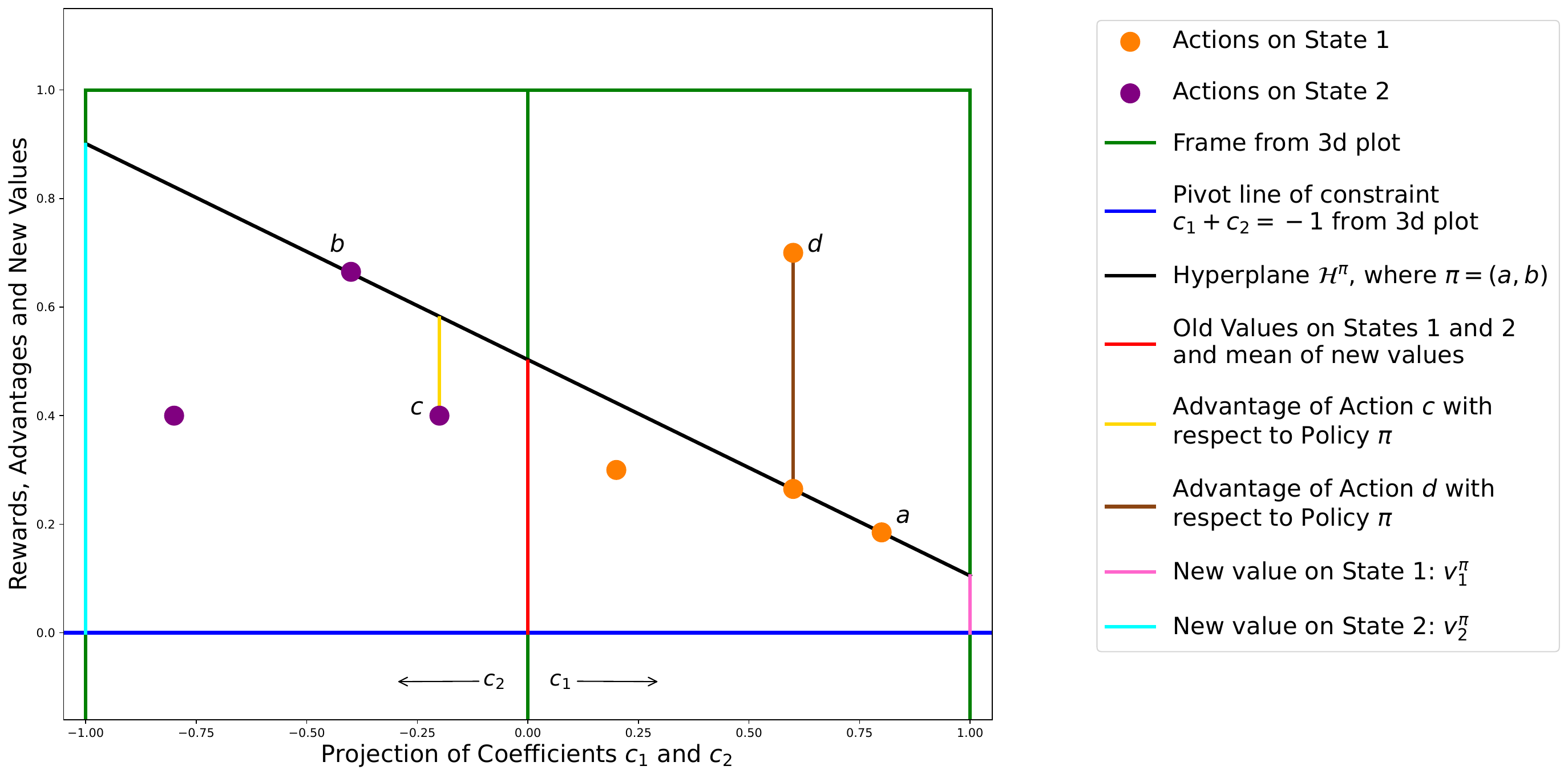}
    \caption{2-State MDP, new interpretation, average reward.}
    \label{fig:new_geom}
  \end{subfigure}
  \caption{Old (left) and new (right) visualizations of a two-state MDP with the same transition probabilities and rewards but different discount factors. The right panel also illustrates why the geometric interpretation developed for the discounted-reward case does not extend directly to the average-reward case: in the latter, the vertical value lines collapse into a single line, so all states have the same value, and this set of values does not define a unique hyperplane. At the same time, a hyperplane can still be constructed for an ergodic policy, and the geometric picture remains logically coherent in that case. To extend the geometric framework to the average-reward case, we propose measuring the values not on the inner but on the outer edges of the action zones. These new values $v$ can be used in both the average- and discounted-reward cases.}
\label{fig:old_and_new}
\end{figure}
In the discounted-reward setting, we can rearrange the Bellman equation to obtain the linear system $(I-\gamma P^\pi) V = R^\pi$, which possesses a unique fixed point $V^\pi$ as $ I-\gamma P^\pi $ is invertible. 
However, this formula breaks down in the average-reward case $\gamma=1$, where the matrix $ I- P^\pi$ is singular, and the discounted value function is no longer well-defined. Instead, the analysis relies on the gain function $\rho^\pi$ and the relative value function $h^\pi$. This can introduce two difficulties, \emph{(i)} the discounted case and average-reward case require different value functions, and \emph{(ii)} the relative value function is not unique, since $h^\pi + c\mathbbm{1}$ is also a solution for any constant $c$. 

This problem can also be seen in the geometric interpretation in Figure~\ref{fig:old_and_new}, where the value function $V^\pi(s)$ is represented by the inner vertical line on the state $s$.
When $\gamma=1$, the vertical lines corresponding to different states coincide, reflecting the degeneracy of the classical value representation in the average-reward case.
Nevertheless, this observation suggests a potential avenue for a
new interpretation. The advantage function remains
well-defined for any value of $\gamma$, including $\gamma=1$.
We develop the geometric interpretation below. The additional algebraic and geometric intuition is provided in Appendix~\ref{app:intuition}.

\paragraph{New Action and Policy Vectors.} 
As suggested by Figure~\ref{fig:new_geom}, we propose measuring values along the outer vertical lines, rather than along the inner vertical lines as previously, which leads to the following definitions of the action and policy vectors.

For SAP $z=(s,a)$, let the \textbf{new SAP vector} be
$
z^+ \!=\! \Bigl(r^z, \tfrac{\gamma p^z_1 - \gamma}{C}, \dots, \tfrac{\gamma p^z_s - 1 - \gamma}{C}, \dots, \tfrac{\gamma p^z_n - \gamma}{C}\Bigr),
$
where $C = n\gamma + (1-\gamma)$ is an MDP constant. For convenience, denote by \(c_i^z\) the \(i\)-th coordinate of \(z^+\) excluding the first element \(r^z\), that is, 
\[
c_i^z =
\begin{cases}
\tfrac{\gamma p^z_i - \gamma}{C}, & \text{if } i \ne s, \\[4pt]
\tfrac{\gamma p^z_s - 1 - \gamma}{C}, & \text{if } i = s.
\end{cases}
\]
Note that the sum of all entries of $z^+$ except the first one equals $-1$. The first entry is special; following \citet{mdp_geometry}, we refer to it as the \(0\)-th coordinate (or \emph{height}).

Given a unichain policy $\pi$, define the state values $v^\pi(s)$ by
\begin{equation} 
    \frac{v^\pi(s)}{C} \;=\; r^{(s,\pi(s))} \;-\; \frac{\gamma\, v^\pi_\Sigma}{C} \;+\; \gamma \sum_{i=1}^n p_i^{(s,\pi(s))} \frac{v^\pi(s_i)}{C},
    \quad\text{where } v^\pi_\Sigma = \sum_{i=1}^n v^\pi(s_i). \label{equ:new v definition}
\end{equation}
Then define the \textbf{new policy vector} \(v^\pi_+\) as
$
v^\pi_+ \;=\; \bigl(1,\; v^\pi(s_1),\; \dots,\; v^\pi(s_n)\bigr).$

Below, we first discuss how this unified value representation leads to a consistent policy evaluation and a natural normalization of the MDP, such that we can verify the existence of the value and reflect the geometric inspiration in an algebraic way.
We then relate these newly introduced quantities to the classical value definitions in both the discounted- and average-reward settings. 
Throughout this section, $V^\pi$ denotes the classical value function, whereas $v^\pi$ denotes the new value representation introduced here.

\paragraph{Policy Evaluation and MDP Normalization.} In the average-reward case, the matrix $I - \gamma P$ is not invertible when $\gamma=1$, so a unique value vector $V^\pi$ satisfying $R^\pi = (I - \gamma P^\pi) V^\pi$ does not exist. However, rewriting Equation \eqref{equ:new v definition} in matrix form and rearranging terms yields 
$
R^\pi = (I + \gamma E - \gamma P^\pi)\,\frac{v^\pi}{C},
$
where $E = \mathbbm{1} \mathbbm{1}^\top$ is the all-ones matrix. Whenever the matrix $I + \gamma E - \gamma P^\pi$ is invertible, the solution admits the explicit representation $v^\pi= C (I+\gamma E-\gamma P^\pi)^{-1}R^\pi.$
We can now give the algebraic definition of these new values:
\begin{definition}[New value function]
For a given policy $\pi$, the proposed value vector $v^\pi \in \mathbb{R}^n$ is defined as the unique solution to the linear system
\begin{align} \label{equ: vector form equ}
    v^\pi= C (I+\gamma E-\gamma P^\pi)^{-1}R^\pi,
\end{align}
where $C = n\gamma + 1 - \gamma$ for $\gamma\in(0,1]$.
\end{definition}

This equation can be solved in the average-reward case (i.e., $\gamma=1$) when the policy $\pi$ induces a unichain Markov chain, as shown in the following lemma. The proof is provided in Appendix~\ref{APP:proof for section 3}.

\begin{lemma} \label{lem:uniqueness for new v}
Let $\pi$ be a policy with transition kernel $P^\pi$ on a finite state space, and let $\mathbbm{1}$ denote the all-ones vector with $E := \mathbbm{1} \mathbbm{1}^\top$. The matrix $I + E - P^\pi$ is invertible if, and only if, the Markov chain induced by $P^\pi$ is unichain.
In particular, in this case, the values $v^\pi$ defined via the linear system $(I + E - P^\pi)\,v^\pi = n R^\pi$ are uniquely determined.
\end{lemma}

Since we can compute values for any unichain policy $\pi$, we can compute
the advantage of any SAP with respect to such a policy, which is
equivalent to constructing the policy hyperplane $\mathcal{H}^\pi$. This,
in turn, implies that the transformation $\mathcal{L}^\Delta$ described
by \citet{mdp_geometry} can be applied to the MDP. Therefore, if the
optimal policy induces a unichain MDP (the case we analyze in this
work), the MDP can be normalized, which significantly simplifies the
analysis.

Note that \textbf{normalization is not an additional assumption}. It is
always possible in the discounted-reward case and is possible in the
average-reward case under the assumptions made in this work. 

Let $v^\pi_{\rm orig}$ denote the values in the original MDP
$\mathcal{M}_{\rm orig}$ and let $v^\pi$ denote the values in the
normalized MDP $\mathcal{M}$. The normalization transformation \citep{mdp_geometry} has two
key properties with respect to VI dynamics (Theorem 3.5):
\begin{itemize}
    \item Since it preserves advantages, the state-wise maximizing SAPs are preserved, and therefore the same greedy policy is selected at every VI iteration. The corresponding value iterates remain related by the same normalization transformation.
    
    \item Since it shifts all values in a state by the same quantity $\delta$,
     differences between policy values are preserved, i.e.,
    for any two policies $\pi$ and $\pi'$,    
    $v^\pi - v^{\pi'} =
    v^\pi_{\rm orig} - v^{\pi'}_{\rm orig}.$  
\end{itemize}

Together, these properties imply that if, after $T$ iterations, VI
produces an $\epsilon$-optimal policy in the normalized MDP, then the
same policy is also $\epsilon$-optimal in the original MDP. Therefore,
it is sufficient to analyze the normalized MDP instead of the original
one.

Overall, normalization does not change the structure of the MDP; it only
simplifies the analysis and makes the geometry more transparent.
Algebraically, the analysis presented in this paper can be formulated in terms of the error
vectors $e^\pi = v^\pi_{\rm orig} - v^*_{\rm orig}$,
and the advantages with respect to the optimal policy, $\adv{z}{v^*}$,
instead of the normalized values $v^\pi$ and rewards $r^z$.

\paragraph{Advantage Representation under the New Values.} 
The geometric interpretation motivates the design of the new value function so that the policy hyperplane remains unchanged and the advantage function is preserved. The following lemma states that under the proposed value representation, the inner product of the SAP vector and the policy vector recovers the classical advantage function.

\begin{lemma} \label{lem:union equivalence of adv fct}
For $\gamma\in(0,1]$, the inner product of SAP vector $z$ and policy vector $\pi$ is equal to the advantage of SAP $z$ with respect to $\pi$, i.e., $z^+v^\pi_+ = r^z + \sum_i c^z_iv^\pi(i) = \adv{z}{\pi}.$
\end{lemma}
The result establishes equivalence between the old and new geometric interpretations in both the discounted and average-reward cases.
Consequently, the geometric interpretation of policy hyperplanes remains valid and can be used directly in VI.
The detailed derivations are provided in Appendix~\ref{APP: relationship details}, where we establish the connections between the new value function and the classical value function, the average reward, and the relative value function.

\section{Main Result} \label{sec:main result}

Thus far, we have established the structural equivalence of the new value function in the discounted and average-reward cases via the geometric interpretation. Next, we investigate the algorithmic consequences. In particular, we analyze the convergence behavior of VI based on the advantage function associated with the new value function $v^\pi$ and show geometric convergence in both the discounted and average-reward cases.
For clarity of exposition, we consider the normalized MDP in the following and detailed proofs of the results in this section, and supporting lemmas are deferred to Appendix~\ref{APP:main results}.
We work under the following assumption, whose feasibility is discussed in Section~\ref{sec:discussion}.
\begin{assumption} \label{ass:irred_and_aper_MDP}
There is a \textbf{unique} optimal policy $\pi^*$. The MDP implied by $\pi^*$ is \textbf{unichain} and \textbf{aperiodic}.
\end{assumption}

Now, we present the \textbf{Value Iteration algorithm} on the new values $v^\pi$. Starting from the arbitrary values $v_0$, the algorithm consists of iteratively performing an update similar to Equation~\eqref{eq:vi_update}:

\begin{equation*}
\frac{v_{t+1}}{C} = \frac{v_{t}}{C} + \max_{z} \adv{z}{v_t} = R^{\pi_t} - E\frac{v_{t\Sigma}}{C}  + \gamma P^{\pi_t}\frac{v_t}{C},
\end{equation*}


where $\pi_t$ is a policy greedy with respect to the values $v_t$, $\pi_t(i) = \argmax_{z, \st(z)=i} \adv{z}{v_t}$. The resulting algorithm, although it chooses a different anchor, shares similarities with the Relative Value Iteration algorithm
\citep{white1963}, which inspired our choice of $v$ as the notation for
the new values.

Then, we show that the new update has a contraction property with respect to $\spanv{v_t}$, the span seminorm of the value vector. 

\begin{theorem} \label{thm:sync_no_lr}
If Assumption \ref{ass:irred_and_aper_MDP} holds, the span of the normalized value vector obtained after $T=n^2$ steps of a standard Value Iteration algorithm satisfies the following inequality:
\begin{equation*}
 \spanv{v_T} \le \gamma^T \iota\, \spanv{v_0},   
\end{equation*} 

where $\gamma\in(0,1],\;\iota \in (0,1)$.
\end{theorem}
Here, $\iota$ is an MDP-dependent contraction factor
($\iota=1-\bar\phi|\mathcal S_R|$ in the proof), arising from the uniform
positive contribution of the recurrent states of the optimal-policy-induced
Markov chain over $T$ iterations. The bound is primarily theoretical and is
generally conservative; we discuss the sources of this conservativeness in Appendix~\ref{app:conservativeness of the results}.

The contraction property holds for all $\gamma\in(0,1]$, implying asymptotic
convergence of value iteration in the span seminorm. For $\gamma\in(0,1)$,
corresponding to the discounted-reward setting, geometric convergence follows
from the classical factor $\gamma^T$ and is further strengthened by the
additional factor $\iota$. For $\gamma=1$, corresponding to the average-reward
setting, geometric contraction is induced entirely by $\iota$.

\begin{proof}
We give the proof strategy here; the full proof is provided in Appendix~\ref{APP:main results}. The proof consists of six steps.

\textbf{Step 1:} Using the properties of the normalized MDP, we derive upper and lower bounds on each component of $\widetilde v_{t+1}$ using the optimal SAP and the current greedy SAP.

\textbf{Step 2:} We combine these bounds to represent the one-step update using a stochastic update matrix:
\begin{equation*}
\widetilde v_{t+1}+\gamma E \widetilde{v}_{t}=\gamma P'_t \widetilde v_{t}, \qquad P'_t=D_tP^\star+(I-D_t)P_t.
\end{equation*}

\textbf{Step 3:} The strict advantage gap implied by uniqueness of the optimal policy gives a uniform positive lower bound on the entries of the diagonal matrix $D_t$, allowing us to decompose
\begin{equation*}
P'_t=\widetilde P'_t+\delta'P^\star,
\qquad \delta'>0.
\end{equation*}

\textbf{Step 4:} We combine the updates over $T$ iterations by introducing the $T$-step kernel
\begin{equation*}
K_T=P'_T\cdots P'_1,
\end{equation*}
and express each component of $\widetilde v_T$ as a weighted combination of the initial values, with coefficients $\lambda_i^s=[K_T]_{s,i}$.

\textbf{Step 5:} Using the unichain and aperiodicity properties of $P^\star$, we obtain a uniform positive lower bound on the coefficients $\lambda_i^s$ associated with its recurrent states. Choosing a strictly smaller common lower bound $\bar\phi>0$ yields a Doeblin minorization of the $T$-step kernel.

\textbf{Step 6:} Finally, we separate this common recurrent-state component from the remaining coefficients. The common component is state-independent and therefore cancels when taking the span, while the remaining coefficients have total weight $1-\bar\phi|\mathcal S_R|<1$. This yields the contraction factor
\begin{equation*}
\iota=1-\bar\phi|\mathcal S_R|\in(0,1).
\end{equation*}
\end{proof}

The following corollaries extend the convergence guarantee obtained for the normalized MDP to the original MDP. Using the correspondence induced by the normalization transformation, they translate the span-convergence bound into $\epsilon$-optimality guarantees and derive the resulting iteration complexity.

\begin{corollary} \label{cor:iteration_complexity_disc}
    In the discounted reward case, the Value Iteration Algorithm outputs an $\epsilon$-optimal policy after 
    \begin{align*}
        \mathcal{O}\left( \frac{\log(1/\epsilon) + \log(1/(1-\gamma))}{\log(1/\gamma) + \frac{\log(1/\iota)}{n^2} } \right)
    \end{align*}
    iterations.
\end{corollary}

\begin{corollary} \label{cor:iteration_complexity_av}
    In the average-reward case, the Value Iteration Algorithm outputs an $\epsilon$-optimal policy after 
    \begin{align*}
        \mathcal{O}\left( \frac{\log(1/\epsilon)}{\frac{\log(1/\iota)}{n^2} } \right)
    \end{align*}
   iterations.
\end{corollary}

\section{Discussion and Conclusion} \label{sec:discussion}

\paragraph{Discussion. Comparison with \citet{lee2025optimal}.} Let us explicitly compare our results to the state-of-the-art analysis presented by \citet{lee2025optimal}, in particular, their Theorem~3.

\begin{theorem}[Theorem 3 from \cite{lee2025optimal}]
Let $t\ge 0$, $n \ge t+2$, and $V_0 \in \mathbb{R}^n$. Then there exists a unichain MDP with $|\mathcal{S}|=n$ and $|\mathcal{A}|=1$ such that its modified Bellman equations has a solution $(g^*, h^*)$ satisfying
\begin{equation*}
\left\lVert \sum_{i=0}^t \beta_i (TV_i - V_i) - g^* \right\rVert_\infty \ge
\frac{1}{t+1} \left\lVert V_0 - h^* \right\rVert
\end{equation*}
for any iterates $\{ V_i \}_{i=0}^t$ satisfying the span condition
\begin{equation*}
V_{t+1} \in V_o + \text{span}\{TV_0 - V_0, \dots, TV_k - V_k\}
\end{equation*}
and any choice of real numbers $\{ \beta_i \}_{i=0}^k$ such that $\sum_{i=0}^k \beta_i = 1$.
\end{theorem}
Then, setting $\beta_k$ to 1 gives the lower bound on the convergence of the VI algorithm.

The theorem is proved by providing an example for which the convergence of 
VI cannot exceed a sublinear rate. The example constructed in 
the proof is a unichain MDP that satisfies Assumption 
\ref{ass:irred_and_aper_MDP}. At first glance, this seems to contradict our 
main result. Does this mean that one of the papers contains a technical 
error?

The answer is No. The apparent contradiction arises from two features of the 
analysis by \citet{lee2025optimal}: the use of the $\ell^\infty$ norm of the 
Bellman error as the performance measure and the consideration of a small 
number of iterations relative to the number of states. For the former, \citet{lee2025optimal} also discuss convergence in 
terms of the span seminorm at the end of Section~2, where they observe that 
convergence in $\ell^\infty$ implies convergence in the span seminorm 
because the span seminorm is upper bounded by twice the $\ell^\infty$ norm. 
This bound shows that the span seminorm cannot converge \emph{slower} than 
the $\ell^\infty$ norm, but it does not rule out that it can
converge \emph{significantly faster}. As we show in this paper, this is 
indeed the case under Assumption~\ref{ass:irred_and_aper_MDP}. Note that, although the infinity norm is commonly used in classical analyses, the span 
seminorm is sufficient here for deriving policy-quality guarantees, as shown in Corollaries 
\ref{cor:iteration_complexity_disc} and \ref{cor:iteration_complexity_av}, 
it suffices for evaluating the quality of the resulting policy.

The second feature concerns the number of algorithm iterations. 
\citet{lee2025optimal} establish the bound for $t \le n-2$ iterations. 
This choice ensures that the sublinear rate follows from an 
information-based complexity argument. Intuitively, there exist two states 
in the graph such that the shortest path between them requires exactly 
$n-1$ steps. If the number of iterations is at most $n-2$, the first state 
cannot receive any information from the last state, and therefore the value 
error may remain arbitrarily large. As we show in this paper, $n^2$ 
iterations are sufficient for any two states to communicate, after which the 
geometric convergence rate becomes visible in finite time.

Taken together, these two features enable an analysis that is technically 
correct, but may lead to the misleading conclusion that VI exhibits only 
sublinear finite-time convergence in the unichain case. In contrast, our 
results show that when the optimal policy is unichain (which always holds in 
the unichain MDP setting), the convergence rate is in fact geometric.

\paragraph{Discussion on Assumption \ref{ass:irred_and_aper_MDP} and the multichain case.}\label{par:assumptio_discussion}
We now discuss the limitations of Assumption \ref{ass:irred_and_aper_MDP}. 
At first glance, this assumption may appear restrictive; however, in the 
average-reward setting, it is often satisfied. 
Aperiodicity can be ensured by introducing self-loops at each state, for example, via a standard lazy transition (mixing with the identity) (Sec .~8.5.4, \cite{Puterman1994}). Such a transformation does not change the policy gains in average-reward MDPs, and therefore does not affect the conclusions of the analysis. 
Then, the prerequisite required to guarantee uniqueness is a strict
preference over states and actions, which can be interpreted as a
tie-breaking rule (e.g., induced by an arbitrarily small perturbation of
the rewards). Under this condition, the unichain optimal policy follows
from a standard connectivity assumption commonly used in exploration
analysis \citep{auer2008near, boone2025logarithmic}. This assumption
states that for any two states $s$ and $s'$, there exists a policy $\pi$
such that $s'$ is reachable from $s$ with positive probability in finite
time, i.e.,
\[
\forall s,s'\ \exists \pi:\ 
\mathbb{P}_\pi\big(\tau_{s'} < \infty \mid s_0 = s\big) > 0,
\]
where $\tau_{s'} = \inf\{t \ge 0 : s_t = s'\}$ is the hitting time of
$s'$. The connectivity assumption is consistent with the unichain setting
and is satisfied in a broad class of MDPs where actions allow transitions
between states.
Thus, the assumption fails only when there exists at least one isolated 
class (i.e., a set of states that cannot be left under any policy) together 
with another class of connected states that achieves a higher gain under any 
policy. In this setting, the same line of analysis cannot be applied, since it
requires sufficiently accurate value estimates to provide guarantees on the
resulting policy. A general analysis of this case is left for future work.

\paragraph{Conclusion.} We analyzed the convergence rates of VI in the discounted and average-reward settings.
Our analysis revealed that, under a unique, unichain optimal policy, VI converges geometrically with a rate strictly faster than $\gamma$ in the span seminorm.
This result significantly improves upon earlier analysis, which suggested that sublinear convergence is optimal in the average-reward setting.
The analysis was based on a geometric interpretation of MDPs, which enabled us to analyze both settings jointly, in contrast to most prior work that analyzes them separately.

\section*{Acknowledgements}
This work was supported by the Research Council of Finland Flagship programme: Finnish Center for Artificial Intelligence FCAI.
We acknowledge the financial support of the Finnish Ministry of Education and Culture through the Intelligent Work Machines Doctoral Education Pilot Program (IWM VN/3137/2024-OKM-4).
\bibliography{main}
\bibliographystyle{rlj}

\newpage
\appendix








\section{Details of the MDP example.} \label{app:mdp_example}

\paragraph{MDP construction (heaven–purgatory–hell).}
We consider a finite MDP $M=(\mathcal S,\mathcal A,P,r)$. The state space
\[
\mathcal S=\{H_1,H_2,H_3,\ P_1,P_2,P_3,\ L_1,L_2,L_3\}
\]
is partitioned into three groups: ``heaven'' $\mathcal S_H=\{H_1,H_2,H_3\}$, ``purgatory'' $\mathcal S_P=\{P_1,P_2,P_3\}$, and ``hell'' $\mathcal S_L=\{L_1,L_2,L_3\}$. We analyze the dynamics under a fixed (optimal) stationary policy $\pi^\star$ that selects a single available action in each state (the construction below specifies the induced transition kernel $P^{\pi^\star}$ and rewards).

\medskip
\noindent\textbf{Rewards.} The per–step reward function is
\[
r(s)=\begin{cases}
1,& s\in\mathcal S_H,\\
0,& s\in\mathcal S_P,\\
-1,& s\in\mathcal S_L.
\end{cases}
\]

\medskip
\noindent\textbf{Transitions under $\pi^\star$.} The induced row–stochastic matrix $P\equiv P^{\pi^\star}$ is specified by the following nonzero probabilities.

\emph{Heaven cycle with rare leakage to purgatory:}
\[
\begin{aligned}
&\Pr(H_2\mid H_1)=0.75,\qquad \Pr(P_1\mid H_1)=0.25,\\
&\Pr(H_3\mid H_2)=0.75,\qquad \Pr(P_1\mid H_2)=0.25,\\
&\Pr(H_1\mid H_3)=0.75,\qquad \Pr(P_1\mid H_3)=0.25.
\end{aligned}
\]

\emph{Purgatory chain that typically returns to heaven, with leakage to hell:}
\[
\begin{aligned}
&\Pr(P_1\mid P_1)=0.20,\quad \Pr(P_2\mid P_1)=0.60,\quad \Pr(L_1\mid P_1)=0.20,\\
&\Pr(P_2\mid P_2)=0.20,\quad \Pr(P_3\mid P_2)=0.60,\quad \Pr(L_1\mid P_2)=0.20,\\
&\Pr(P_3\mid P_3)=0.20,\quad \Pr(H_1\mid P_3)=0.60,\quad \Pr(L_1\mid P_3)=0.20.
\end{aligned}
\]

\emph{Hell cycle that occasionally escapes to purgatory:}
\[
\begin{aligned}
&\Pr(P_2\mid L_1)=0.50,\qquad \Pr(L_2\mid L_1)=0.50,\\
&\Pr(P_2\mid L_2)=0.50,\qquad \Pr(L_3\mid L_2)=0.50,\\
&\Pr(P_2\mid L_3)=0.50,\qquad \Pr(L_1\mid L_3)=0.50.
\end{aligned}
\]

\medskip
\noindent\textbf{Summary.} Under $\pi^\star$, the heaven states form a $3$-cycle with reward $1$ and a $0.25$ chance per step to fall into purgatory; purgatory forms a directed chain that either advances, self–loops with probability $0.20$, leaks to hell with probability $0.20$, or returns to heaven from $P_3$ with probability $0.60$; hell is a $3$–cycle with reward $-1$ and a $0.50$ chance per step to move to $P_2$. This fully specifies $(\mathcal S,P,r)$ for the single–action (policy–fixed) analysis of this example.

\section{Relationship between the New Value and Classical Value}\label{APP: relationship details}
In this appendix, we provide the technical details underlying the connections between the proposed value function and the classical formulations in both the discounted and the average-reward settings. We first present the main relationships and then provide the proofs of the lemmas.

\paragraph{Discounted Reward Case}
From a geometric perspective, the difference between the new values and the classical values is the different locations of the vertical value lines on the states, as shown in Figure~\ref{fig:first} and Figure~\ref{fig:new_geom}. To illustrate the algebraic interpretation of the new value function, we show the relationships between the old and new values in the discounted reward case. 

When $\gamma <1$ for the discounted reward case, we first construct the connection with the old value function.

\begin{lemma} \label{lem:discounted value relationship}
    If $\gamma<1$, the for a given $\pi$, there exists
    \begin{align} \label{equ:discounted case relationship between old and new v}
        V^\pi(s) = \frac{v^\pi(s)}{C} + \frac{\gamma v^\pi_{\sum} }{C(1-\gamma)}.
    \end{align}
    Consequently, the sums of the value function exists as  
    \begin{align} \label{equ:discounted case sum relationship between old and new v}
        v^\pi_\Sigma = (1-\gamma)V^\pi_\Sigma,
    \end{align}
    where $ V^\pi_\Sigma = \sum_s V^\pi(s)  $, and for the mean value $ \bar V^\pi := \frac{V^\pi_\Sigma}{n}$ and $ \bar v^\pi := \frac{v^\pi_\Sigma}{n} $, there exists
    \begin{align*}
        V^\pi(s) - \bar{V}^\pi = \frac{1}{C} ( v^\pi(s) - \bar{v}^\pi ).
    \end{align*}
\end{lemma}
This Lemma establishes the connection between the new value and the classical value in both component-wise and sum-wise forms. From the sum perspective, the factor $1-\gamma$ is introduced into the equivalence, which coincides with the geometric interpretation that, in the classical case, there exists a scaling weight $1-\gamma$ that reflects the value function on the vertical axis, whereas in the new case this scaling weight is removed. 

With this, we can finally establish the equivalence between the old and new geometric interpretations in the discounted case.

\begin{lemma} \label{lem:discounted case equivalence}
    
In the discounted case, the inner product of new SAP vector $z^+$ and new policy vector $v^\pi_+$ is equal to the advantage of SAP $z$ with respect to values $v^\pi$:

\begin{align*}
z^+v^\pi_+ &= r^z + \sum_i c^z_iv^\pi(i) = \adv{z}{\pi}.
\end{align*}
\end{lemma}

Thus, in the discounted case, the classic value iteration algorithm can still be implemented as the new value function does not change the advantage function. Furthermore, in our newly defined value function, $\gamma$ can be $1$, which motivates investigating equivalence in the average-reward case. 

\paragraph{Average Reward Case.} The equivalence of the new value function and the old value function in the discounted case has been established. 
We recall the equivalence between the discounted value function and the relative value function from the literature in the following lemma.

\begin{lemma}[Corollary 8.2.4 \cite{Puterman1994}]\label{lemma:old equivalence}
    Given any policy $\pi$ and state $s$,
    \[
        V^\pi(s) \;=\; \frac{\rho^\pi(s)}{1-\gamma} \;+\; h^\pi(s) \;+\; f^\pi(s,\gamma),
    \]
    where $V^\pi$ and $h^\pi$ are the old discounted value function and the old relative value function, respectively; $\rho^\pi(s)$ is the average reward or gain, which is independent of $s$ under standard unichain assumptions; and $f^\pi$ is a remainder term such that $\lim_{\gamma\to 1} f^\pi(s,\gamma) = 0$.
\end{lemma}

As a result of $ v_\Sigma^\pi = (1-\gamma)V^\pi_\Sigma $ shown in Lemma~\ref{lem:discounted value relationship} for the discounted case, we can get the equivalence between the traditional average reward and the new defined value function when considering the average-reward MDP. 
\begin{lemma} \label{lemma:average equivalence}
    If $\gamma=1$, then for a given $\pi$,
    \[ \rho^\pi = \frac{v^\pi_\Sigma}{n} .\] 
\end{lemma}
Lemma~\ref{lemma:average equivalence} establishes the connection between the new value function and the average reward. In this way, the value function is still valid when applying $\gamma=1$, and the average reward can be explicitly expressed by the value. 

Let $\gamma=1$ in our new value function in Equation~\eqref{equ:new v definition}. Since $C=n$, Lemma~\ref{lemma:average equivalence}, yields 
\[
\frac{v^\pi(s)}{C} \; +\;  \rho^\pi\;=\; r^{\pi(s)} \;+\;\sum_{i=1}^n p_i^{\pi(s)} \frac{v^\pi(i)}{C}.
\]
This new value function with MDP constant $C$ can satisfy the Bellman equation in the average reward MDP:
\[
T^\pi \frac{v^\pi(s)}{C} \;= \; r^{\pi(s)} \;+\; \sum_{i=1}^n p_i^{\pi(s)} \frac{v^\pi(i)}{C} = \frac{v^\pi(s)}{C}\; + \;\rho^\pi .
\]
Our new value function $\frac{v^\pi(s)}{C}$ is unique to an additive constant, hence $\frac{v^\pi(s)}{C} \in \{h^\pi+c\mathbbm{1}\mid c\in\mathbbm{R}\}$. We then build an exact relationship between the new value function and the old relative value function.

\begin{lemma} \label{lem:relative value function relationship}
    For a given policy $\pi$, there exists
    \begin{align*}
        h^\pi(s) = \frac{v^\pi(s)}{n} - \frac{\rho^\pi}{n},
    \end{align*}
    with $\sum_s  h^\pi(s) = 0$.
\end{lemma}
In the average-setting, the bias function is defined only up to an additive constant, since if $h^\pi$ is a solution of the Bellman equation, then $ h^\pi+c\mathbbm{1} $ is also a solution for any constant $c$. 
In contrast, the proposed value $v^\pi$ is defined via the linear system and uniquely determined, as shown in Lemma~\ref{lem:uniqueness for new v}. 
The relation in Lemma~\ref{lem:relative value function relationship} therefore identifies the unique representative of the affine solution set $\{h^\pi+c\mathbbm{1}\mid c\in\mathbbm{R}\}$ that satisfies the condition $\sum_{s} h^\pi(s)=0$.

We now turn to the action-level characterization. In particular, we show that the inner product between the SAP vector and the policy vector recovers the advantage function under policy $\pi$ in the average-reward setting.
\begin{lemma}\label{lem:average reward case geometric equivalence with adv}
    In the average reward case, the inner product of the SAP vector $z^+$ and the policy vector $v^\pi_+$ is equal to the advantage of SAP $z$ with respect to policy $\pi$:
\begin{align*}
z^+v^\pi_+ &= r^z+ \sum_i c^z_iv^\pi(i) = \adv{z}{\pi}.
\end{align*}
\end{lemma}
In this section, we proposed that the newly defined value function remains well-defined in both the discounted and average-reward setting. Under the unichain assumption, the vector $v^\pi/n$ satisfies the average-setting Bellman equation and therefore corresponds to the canonical bias function.
Moreover, the geometric inner product between SAP vectors and policy vectors coincides exactly with the classical advantage function.
This established the consistency of the proposed value function across the discounted and average reward cases.

In the following, we provide the proofs of the lemmas in this section.

\begin{proof}[Proof of Lemma~\ref{lem:discounted value relationship}]
    If $\gamma < 1$, we can subtract the quantity $\frac{\gamma v^\pi_\Sigma}{C(\gamma-1)} $ from both sides of the value equation Equation~\ref{equ:new v definition}, which leads to: 

\begin{align*}
\frac{v^\pi(s)}{C} - \frac{\gamma v^\pi_\Sigma}{C(\gamma-1)} &= r^{\pi(s)} - \frac{\gamma v^\pi_\Sigma}{C} - \frac{\gamma  v^\pi_\Sigma}{C(\gamma-1)} + \gamma \sum_i p_i^{\pi(i)} \frac{v^\pi(i)}{C}\\
&= r^{\pi(s)} - \frac{\gamma v^\pi_\Sigma((\gamma-1) + 1)}{C(\gamma-1)}  + \gamma \sum_i p_i^{\pi(i)} \frac{v^\pi(i)}{C}  \\
&= r^{\pi(s)}  + \gamma \sum_i p_i^{\pi(i)} \left(\frac{v^\pi(i)}{C}- \frac{\gamma v^\pi_\Sigma}{C(\gamma-1)} \right).
\end{align*}

Therefore, the values $\frac{v^\pi(s)}{C} - \frac{\gamma v^\pi_\Sigma}{C(\gamma-1)}$ satisfy the classic Bellman equation and, due to uniqueness, 
\begin{align*} 
    V^\pi(s) = \frac{v^\pi(s)}{C} - \frac{\gamma v^\pi_\Sigma}{C(\gamma-1)}=
\frac{(1-\gamma)v^\pi(s)}{C(1-\gamma)} + \frac{\gamma v^\pi_\Sigma}{C(1-\gamma)},  
\end{align*}
which indicates the relationship between the old value function $V^\pi$ and the new defined value $v^\pi$ in the discounted case. Then, we can sum this expression over all states to get the sum relationship:
\begin{align*} 
   \sum_s V^\pi(s) = \sum_s \left[ \frac{(1-\gamma)v^\pi(s)}{C(1-\gamma)} + \frac{\gamma v^\pi_\Sigma}{C(1-\gamma)} \right] = \frac{(1-\gamma)v^\pi_\Sigma}{C(1-\gamma)} + \frac{\gamma nv^\pi_\Sigma}{C(1-\gamma)} = \frac{v^\pi_\Sigma}{1-\gamma}. 
\end{align*}
Defining the mean values as $\bar{V}^\pi = (\sum_i V^\pi(s))/n$ and similarly for $\bar{v}^\pi$, we obtain
\begin{align*}
V^\pi(s) - \bar{V}^\pi &= V^\pi(s) - \frac{\sum_i V^\pi(i)}{n} = 
\frac{(1-\gamma)v^\pi(s)}{C(1-\gamma)} + \frac{\gamma v^\pi_\Sigma}{C(1-\gamma)} - 
\frac{v^\pi_\Sigma}{(1-\gamma)n}\\
&=\frac{v^\pi(s)}{C} + \frac{\gamma n v^\pi_\Sigma}{Cn(1-\gamma)} - 
\frac{v^\pi_\Sigma (n\gamma + (1-\gamma))}{Cn(1-\gamma)} =\\
&=\frac{1}{C}(v^\pi(s) - \bar{v}^\pi).
\end{align*}

\end{proof}

\begin{proof}[Proof of Lemma~\ref{lem:discounted case equivalence}]
We work with the expression $\sum_i - c^z_iv^\pi(i)$, which is the height at which the hyperplane $\mathcal{H}^\pi$ crosses the vertical line with coordinates the same as $z$,

\begin{align*}
\sum_i -c^z_iv^\pi(i) &= \sum_i -c^z_iv^\pi(i) + \bar{v}^\pi - \bar{v}^\pi=
\bar{v}^\pi + \sum_i -c^z_i (v^\pi(i) - \bar{v}^\pi) \\
&=(1-\gamma)\bar{V}^\pi + C\sum_i -c^z_i (V^\pi(i) - \bar{V}^\pi) \\
&= (1-\gamma)\bar{V}^\pi + (V^\pi(s) -\bar{V}^\pi)
+ \sum_i (\gamma - \gamma p^z_i) (V^\pi(i) - \bar{V}^\pi)\\
&=-\sum_i \tilde{c}^z_iV^\pi(i).
\end{align*}
\end{proof}

\begin{proof}[Proof of Lemma~\ref{lemma:average equivalence}] \label{proof: lem average equivalence}
    From Lemma~\ref{lemma:old equivalence} we have
    \begin{align*}
        v^{\pi}_{\Sigma}=(1-\gamma) \sum_{s} V^\pi(s) &= \sum_{s} \rho^\pi(s) + (1-\gamma)\sum_{s}h^\pi(s)+(1-\gamma)\sum_{s}f^\pi(s,\gamma).
    \end{align*}
    If $\gamma=1$, the last two terms converge to $0$. Then $v^{\pi}_{\Sigma} = \sum_{s} \rho^\pi$ and $C =|S|=n$. Therefore,
    \[ \rho^\pi = \frac{v^\pi_\Sigma}{n} = \frac{v^\pi_\Sigma}{C} .\]
\end{proof}

\begin{proof}[Proof of Lemma~\ref{lem:relative value function relationship}] \label{proof:lem relative value function relationship}

For the relative function $h^\pi(s)$, by the equivalence from Lemma~\ref{lemma:old equivalence} and the equivalence for value function in discounted MDP cases $V^\pi(s) = \frac{v^\pi(s)}{C} - \frac{\gamma v^\pi_\Sigma}{C(\gamma-1)}$, there is
\begin{align} \label{equ: relative equ1}
    h^\pi(s) \;=&\;V^\pi(s) - \frac{\rho^\pi(s)}{1-\gamma}  - f^\pi(s,\gamma) \\ \nonumber
    =&\; \frac{v^\pi(s)}{C} - \frac{\gamma v^\pi_\Sigma}{C(\gamma-1)} 
    - \frac{\rho^\pi(s)}{1-\gamma} 
    - f^\pi(s,\gamma) \\ \nonumber
    =&\; \frac{v^\pi(s)}{C} + \frac{\frac{\gamma v^\pi_\Sigma}{C} - \rho^\pi(s)}{1-\gamma}  
     - f^\pi(s,\gamma) 
\end{align}
When taking the limit for $\gamma\to1$, we need to think about the convergence rate for the term $ \frac{\gamma v^\pi_\Sigma}{C}$ as $C = n\gamma+1-\gamma$ is still a function of $\gamma$. Define $ \phi(\gamma) = \frac{\gamma v^\pi_\Sigma}{n\gamma+1-\gamma} $. According to Taylor's expansion around $1$, there is
$$ \phi(\gamma) = \phi(1) + \phi'(1)(\gamma-1) + \mathcal{O}(\gamma-1), $$
where $\phi(1) = \frac{ v^\pi_\Sigma}{n} = \rho^\pi$ and $ \phi'(1) = \frac{(n\gamma+1-\gamma)v^\pi_\Sigma - \gamma(n-1)v^\pi_\Sigma}{(n\gamma+1-\gamma)^2} = \frac{v^\pi_\Sigma}{n^2} $. Plugging this expansion into~\eqref{equ: relative equ1} for $\gamma\to 1$, the equation reduces to
\begin{align*}
    h^\pi(s)=&\; \frac{v^\pi(s)}{n}+\lim_{t\to1}\frac{\phi'(1)(\gamma-1) + \mathcal{O}(\gamma-1)}{1-\gamma}.
\end{align*}
Since $\lim_{\gamma\to1}\frac{\mathcal{O}(\gamma-1)}{1-\gamma} = 0$, we obtain $ h^\pi(s)=\frac{v^\pi(s)}{n}-\phi'(1)= \frac{v^\pi(s)}{n}-\frac{\rho^\pi}{n}.$ $\phi'(1)$ is a constant and $\sum_{s} h^\pi(s) = 0 $. 

\end{proof}

\begin{proof}[Proof of Lemma~\ref{lem:average reward case geometric equivalence with adv}]
    For $\gamma=1$ we have $C=n$ and $\frac{v^\pi(s)}{n} = h^\pi (s) + \frac{\rho^\pi}{n}$. Let $c:=\frac{\rho^\pi}{C}$, then, 
\begin{align*}
     z^+v^\pi_+ & = r^z+ \sum_i c^z_iv^\pi(i)\\
     &=r^{z} - \frac{ v^\pi_\Sigma}{C} + \sum_{i=1}^n p_i^{z} \frac{v^\pi(i)}{C} - \frac{v^\pi(s)}{C}\\
     &= r^{z} - \rho^\pi + \sum_{i=1}^n p_i^{z} \left( h^\pi(i) + c\right) - \left( h^\pi(s) + c\right) \\
     &=r^{z} - \rho^\pi +\sum_{i=1}^n p_i^{z}  h^\pi(i) -h^\pi(s)\\
     & = \adv{z}{\pi}.
\end{align*}
\end{proof}

\section{Proof for Section \ref{sec:new geo interpr with lemmas}}
\label{APP:proof for section 3}
\begin{proof}[Proof of Lemma~\ref{lem:uniqueness for new v}]
($\Rightarrow $) Assume the Markov chain is unichain. Then, $\ker(I - P^\pi) = \mathrm{span}\{\mathbbm{1}\}$ for the (right) kernel. For any vector $x$, write $x = \alpha \mathbbm{1} + y$ with $\mathbbm{1}^\top y = 0$. Using $P^\pi \mathbbm{1} = \mathbbm{1}$ and $E y = (\mathbbm{1}^\top y)\mathbbm{1} = 0$,
\[
(I + E - P^\pi)(\alpha \mathbbm{1} + y)
= \underbrace{(I - P^\pi)\alpha \mathbbm{1}}_{=0}
  + \underbrace{E \alpha \mathbbm{1}}_{=\alpha n \mathbbm{1}}
  + (I - P^\pi) y
  + \underbrace{E y}_{=0}
= \alpha n \mathbbm{1} + (I - P^\pi) y .
\]
If $(I + E - P^\pi)(\alpha \mathbbm{1} + y)=0$, left-multiplying by $\mu^T$ gives $\alpha n = 0$, hence $\alpha=0$. Thus, $(I - P^\pi)y = 0$ with $\mathbbm{1}^\top y=0$. By $\ker(I - P^\pi)=\mathrm{span}\{\mathbbm{1}\}$, we get $y=0$. Hence, $x=0$ is the only solution, so $I + E - P^\pi$ is invertible.

($\Rightarrow $) Assume $I + E - P^\pi$ is invertible. Suppose, for contradiction, that the chain is not unichain. Then there are at least two closed irreducible classes, and hence $\dim \ker(I - P^\pi) \ge 2$. Therefore, there exists a nonzero $y \in \ker(I - P^\pi)$ with $\mathbbm{1}^\top y = 0$ (choose a nontrivial linear combination of basis vectors of the kernel with zero sum). For this $y$, we have $E y = 0$ and
\[
(I + E - P^\pi) y \;=\; (I - P^\pi) y \;+\; E y \;=\; 0,
\]
which contradicts invertibility. Hence, the chain must be unichain.

The final statement follows since invertibility of $I + E - P^\pi$ guarantees a unique solution $v^\pi$ to the corresponding linear system.
\end{proof}

\section{Proofs for Section \ref{sec:main result}  } \label{APP:main results}

In this section, we provide the proofs of Theorem \ref{thm:sync_no_lr} and Corollaries \ref{cor:iteration_complexity_disc} and \ref{cor:iteration_complexity_av}. We precede the proof of the theorem with the following lemma.

\begin{lemma} \label{lemma:product expansion}
    For any row-stochastic matrices $P,P_1,\cdots,P_T\in \mathbb{R}^{n\times n}$ and the all-ones matrix $E\in \mathbb{R}^{n\times n}$, there is 
    $$ \prod_{t=1}^T (P_t  -  E)= \prod_{t=1}^T P_t + \tilde E, $$
    where $\tilde E$ is a matrix with identical rows and $\tilde Ev$ is a vector with identical entries for any vector $v$.
\end{lemma}

\begin{proof}[Proof of Lemma~\ref{lemma:product expansion}]

We first consider any row-stochastic matrix $P$, any matrix $H$ whose rows are all identical and equal to a row vector $h$, and any vector $v$. Then:
\begin{itemize}
    \item $PE = E$.
    \item $EP$ is a matrix with identical rows. Moreover, writing $H = [\,h,\ldots,h\,]^\top$ with $h$ a row vector, we have
    $Hv = [\,hv,\ldots,hv\,]^\top$, i.e., a vector whose entries are all the same scalar $hv$.
    \item For any $m\in\mathbb{N}$, $E^m = n^{m} E$. Hence, $E^m v$ is a vector with identical entries.
\end{itemize}

Thus, for any $i_1,i_2,\cdots,i_k\in\{1,2,\cdots,n\}$,  we have
\begin{itemize}
    \item $ P_{i_1}P_{i_2} \cdots P_{i_k} E = E $ as $PE = E$.
    \item The matrix $ E P_{i_1}P_{i_2} \cdots P_{i_k} $ has identical rows. We note $EP_{i_1} = [h,\cdots,h]^\top$ since it has identical rows $h$. Then, for $P_{i_2}=[v_1,\cdots,v_n]$ with column vectors $v_1,\cdots,v_n $, $EP_{i_1}P_{i_2} = [h,\cdots,h]^\top [v_1,\cdots,v_n] $ shows that $EP_{i_1}P_{i_2}$ has identical rows $ [hv_1, hv_2, \cdots, hv_n]$. By this induction, we obtain the result. 
    \item $  P_{i_1}P_{i_2} \cdots P_{i_j} E  P_{i_{j+1}}P_{i_{j+2}} \cdots P_{i_k} $ has identical rows. 
\end{itemize}

The product of matrices $\prod_{t=1}^T (P_t  -  E)$ can be expanded as
\[
\prod_{t=1}^T (P_t - E)
= \sum_{S \subseteq \{1, \dots, T\}} 
(-1)^{|S|} 
\left( 
\prod_{t \in S} E 
\prod_{t \notin S} P_t 
\right),
\]
where each subset \( S \subseteq \{1, \dots, T\} \) indicates the indices
where \( E \) is selected instead of \( P_t \).
The first term corresponds to \( \prod_{t=1}^T P_t \), and the following
terms 
\[
\sum_{k=1}^T (-1)^k 
\sum_{1 \le i_1 < \cdots < i_k \le T} 
P_T \cdots P_{i_k+1} E P_{i_k-1} \cdots P_{i_1+1} E P_{i_1-1} \cdots P_1
\]
contain one or more \( E \) factors. The product of $E$ with any stochastic matrices from the left can be simplified to $E$. Consequently, each term in the expansion that contains $k$ matrices $E$ can be reduced to $E^k$ multiplied by the remaining right-hand stochastic matrices, which all have identical rows as shown above. Moreover, the sum of matrices with identical rows also has identical rows, denoted as $\tilde E$.

\end{proof}

With Lemma \ref{lemma:product expansion}, we are ready to prove the main theorem:

If Assumption \ref{ass:irred_and_aper_MDP} holds, the span of the normalized value vector obtained after $T=n^2$ steps of a standard Value Iteration algorithm satisfies the following inequality:
\begin{equation*}
 \spanv{v_T} \le \gamma^T \iota\, \spanv{v_0},   
\end{equation*} 

where $\gamma\in(0,1],\;\iota \in (0,1)$.

\begin{proof}[Proof of Theorem~\ref{thm:sync_no_lr}]
\textbf{Step 1:}
   First, note that we analyze a normalized MDP. Therefore, all actions lie on or below the horizontal zero-level hyperplane or, in algebraic terms, SAPs participating in the unique optimal policy $\pi^*$ have $0$ rewards, and SAPs that do not participate have negative rewards.

    For clarity, we define $ \tilde{v}_{t} = \frac{v_{t}}{C} $. 
    With a slight abuse of notation, we write
\[
\operatorname{adv}(z,\tilde v_t)
:=
r^z+\gamma\sum_i(p_i^z-1)\tilde v_t(i)-\tilde v_t(s),
\]
which equals $\operatorname{adv}(z,v_t)$ under the original parameterization.
    Denote the SAP chosen by the optimal policy in the state $s$ as $z^*$ and the SAP greedy with respect to the values $\tilde{v}_{t}$ as $z$ ($z^*\in\pi^*$, $z := \arg\max_{z'} \adv{z'}{\tilde{v}_{t}} $ with $st(z^*)=st(z)=s$). Then, we have
    \begin{align*}
        \tilde{v}_{t+1}(s) = \tilde{v}_t(s)+ \adv{z}{\tilde{v}_t} = r^{z} + \gamma \sum_i (p_i^{z}-1) \tilde{v}_t(i)  \leq \gamma \sum_i (p_i^{z}-1) \tilde{v}_t(i),
    \end{align*}
    since all rewards in a normalized MDP are non-positive. Moreover,
    \begin{align*}
        \tilde{v}_{t+1}(s) = \tilde{v}_t(s)+ \adv{z}{\tilde{v}_t} \geq \tilde{v}_t(s)+ \adv{z^*}{\tilde{v}_t} = \gamma \sum_i (p_i^{z^*}-1) \tilde{v}_t(i),
    \end{align*}
    since the rewards of all optimal SAPs in a normalized MDP are 0. 
    
    \textbf{Step 2:} Combining these bounds, we obtain 
    \begin{align*}
& \gamma \sum_i (p^{z^*}_{i} - 1)\tilde{v}_t(i) \;\leq\; \tilde{v}_{t+1}(s) \;\leq\; \gamma \sum_i (p^z_{i} - 1)\tilde{v}_t(i), \\[6pt]
& \Rightarrow \gamma (P^*-E) \tilde{v}_t \;\leq\; \tilde{v}_{t+1} \;\leq\; \gamma (P_t - E) \tilde{v}_{t} ,\\[6pt]
& \Rightarrow \gamma P^*\tilde{v}_t \;\leq\;\tilde{v}_{t+1}+ \gamma E \tilde{v}_{t} \;\leq\; \gamma P_t \tilde{v}_{t} ,\\[6pt]
& \Rightarrow \tilde{v}_{t+1}+ \gamma E \tilde{v}_{t} =\gamma P'_t \tilde{v}_{t}= \gamma \left[D_t P^* + (I-D_t) P_t\right] \tilde{v}_{t},
\end{align*}
where $E$ is the all-ones matrix and $P'_t:=D_t P^* + (I-D_t) P_t$ represents a coordinate-wise convex combination of the upper and lower bounds. The diagonal matrix $D_t \in \mathbb{R}^{n \times n}$ has the weights of this combinations, \(D_t(s,s) \in [0,1]\). When the optimal action is chosen, the two bounds coincide, and we set
$D_t(s,s)=1$. 
Besides, $ E \tilde{v}_{t}$ is a vector with same entries, therefore for $\gamma\leq1$,
\begin{equation} \label{equ:spv_tp1_spv_t}
    \spanv{\tilde v_{t+1}} = \gamma \spanv{P'_t \tilde{v}_{t}} \leq \gamma \spanv{\tilde{v}_{t}} \leq \spanv{\tilde{v}_{t}},
\end{equation}
where the first inequality holds because, for any state $s$ there exists $\min_i \tilde{v}_t(i) \leq P'_t \tilde{v}_{t} (s)\leq \max_i \tilde{v}_t(i)$.

\textbf{Step 3:} Define the maximum advantage of non-optimal actions with respect to the optimal policy as $-\delta$:
\begin{align*}
    -\delta = \max_{z' \notin \pi^*} \adv{z'}{\pi^*}, \quad \delta > 0.
\end{align*} 
Hence,
\begin{align*}
    \adv{z'}{\pi^*}\leq -\delta < 0, \quad \forall z' \notin \pi^*.    
\end{align*}
Since $ r^z = \adv {z}{\pi^*} \leq -\delta <0$, we obtain a stricter upper bound for $\tilde{v}_{t+1}(s)$ for a state $s$ in which a non-optimal action is chosen:
\begin{align*}
\tilde{v}_{t+1}(s) &= r^z + \gamma \sum_i (p^z_{i}-1) \tilde{v}_t(i)\leq -\delta + \gamma \sum_i (p^z_{i}-1) \tilde{v}_t(i)  < \gamma \sum_i (p^z_{i}-1) \tilde{v}_{t}(i). \\[6pt]
\end{align*}
Then, for such a state
\begin{align*}
    \tilde{v}_{t+1}(s) + \gamma \tilde{v}_{t,\Sigma} &= \gamma \left(D_t(s,s) \sum_i p^{z^*}_{i} \tilde{v}_t(i) + (1-D_t(s,s))\sum_i p^{z}_{i} \tilde{v}_t(i) \right) \\
    &\leq -\delta +\gamma \sum_i (p^z_{i}-1) \tilde{v}_t(i) + \gamma \tilde{v}_{t,\Sigma} \\
    &\leq -\delta+\gamma\sum_i p^z_{i} \tilde{v}_t(i).
\end{align*}
which implies 
\[
\gamma D_t(s,s) \sum_i (p^{z^*}_i - p^z_i)\tilde{v}_t(i) \le -\delta.
\]
We further have 
\begin{equation} \label{equ:spv_Dt}
    -\spanv{\tilde{v}_t} = \min_{i}\tilde{v}_t (i) - \max_{i}\tilde{v}_t(i) \leq \sum_i \left(p^{z^*}_{i} - p^z_{i}\right)\tilde{v}_t(i)\leq -\frac{\delta}{\gamma D_t(s,s)}\le -\frac{\delta}{\ D_t(s,s)},
\end{equation}
where the first inequality follows from $\max_i\tilde{v}_t(i) \geq \sum_i p^z_{i}\tilde{v}_t(i)$ and $\min_i\tilde{v}_t(i) \leq \sum_i p^{z^*}_{i}\tilde{v}_t(i)$. 

Thus, for all states, \eqref{equ:spv_tp1_spv_t} and \eqref{equ:spv_Dt} gives
$$D_t(s,s)\ge \min\!\left[\frac{\delta}{\spanv{\tilde{v}_t}}, 1\right] \ge  
\min\!\left[\frac{\delta}{\text{sp}(\tilde{v}_0)}, 1\right]
= \delta',$$
where we set $\delta'=1$ if $\spanv{\tilde v_0}=0$. 

This allows us to represent the update matrix $P'_t$ as
\begin{equation}
\label{eq:P_decomposition}
P'_t
=
\tilde P'_t+\delta'P^*,
\end{equation}
where
\[
\tilde P'_t
:=
(D_t-\delta'I)P^*
+
(I-D_t)P_t.
\]
Since $D_t-\delta'I$ and $I-D_t$ are nonnegative diagonal matrices,
$\tilde P'_t$ is nonnegative, and each of its rows sums to
$1-\delta'$. Hence, $\tilde P'_t$ is substochastic, and $\delta' P^*$ is a component that remains constant
across iterations.

\textbf{Step 4:} Next, we analyze the value vector after $T=n^2$ iterations. Define
\[
K_T:=P'_T P'_{T-1}\cdots P'_1.
\]
By Lemma~\ref{lemma:product expansion},
\begin{align*}
\tilde v_T
=
\gamma^T K_T\tilde v_0
-
\gamma^T\tilde E\tilde v_0,
\end{align*}
where $\tilde E$ has identical rows.

Define $\lambda_i^s:=[K_T]_{si}.$
Since $K_T$ is stochastic,
\[
\lambda_i^s\geq0,
\qquad
\sum_i\lambda_i^s=1.
\]
Thus, $\lambda_i^s$ is the weight with which $\tilde v_0(i)$
contributes to $\tilde v_T(s)$. Denoting the $k$th entry of each row
of $\tilde E$ by $\tilde E_k$, we have
\[
\tilde v_T(s)
=
\gamma^T\sum_i\lambda_i^s\tilde v_0(i)
-
\gamma^T\sum_k\tilde E_k\tilde v_0(k).
\]


We now use the fact that the policy is unichain. Let us denote the set of states in the recurrent class as $\mathcal{S}_R$ and the set of transient states as $\mathcal{S}_T$. Then, we can rewrite the expression above as:
\begin{align*}
\tilde{v}_{T}(s)
=
\gamma^T \sum_{i \in \mathcal{S}_R}
\lambda_i^s \tilde{v}_0(i)
+
\gamma^T \sum_{j \in \mathcal{S}_T}
\lambda_j^s \tilde{v}_0(j)
-
\gamma^T \sum_k \tilde E_k \tilde{v}_0(k).
\end{align*}

\textbf{Step 5:} It remains to lower bound the coefficients $\lambda_i^s$ associated
with recurrent states.
Using \eqref{eq:P_decomposition},
\[
K_T
=
P'_T\cdots P'_1
=
\prod_{t=1}^{T}
\left(\tilde P'_t+\delta'P^*\right).
\]
All terms in this expansion are nonnegative. In particular, the term
obtained by selecting $\delta'P^*$ from every factor is
\[
(\delta')^T(P^*)^T.
\]
Therefore, entrywise,
\[
K_T
\geq
(\delta')^T(P^*)^T.
\]


By the choice of $T$ and the unichain and aperiodicity properties of
$P^*$, for every $s\in\mathcal S$ and every
$i\in\mathcal S_R$, the entry $[(P^*)^T]_{si}$ is positive.
Let $P_{\min}^*$ denote the minimum nonzero entry of $P^*$.
Then,
\[
\lambda_i^s
=
[K_T]_{si}
\geq
(\delta')^T[(P^*)^T]_{si}
\geq
(\delta'P_{\min}^*)^T
=: \phi>0,
\qquad i\in\mathcal S_R.
\]

Since $K_T$ is stochastic,
$\sum_i\lambda_i^s=1$. Therefore,
\[
|\mathcal S_R|\phi
\leq
\sum_{i\in\mathcal S_R}\lambda_i^s
\leq1.
\]
Hence,
\[
0<|\mathcal S_R|\phi\leq1.
\]

Fix any $\epsilon\in(0,\phi)$ and define
\[
    \bar\phi:=\phi-\epsilon\in(0,\phi).
\]
Since $\lambda_i^s\geq\phi>\bar\phi$ for every
$i\in\mathcal S_R$, we have
\[
    0<|\mathcal S_R|\bar\phi<1.
\]

Let $\nu$ denote the uniform distribution over the recurrent class,
\[
\nu(i)=
\begin{cases}
1/|\mathcal S_R|, & i\in\mathcal S_R,\\
0, & i\notin\mathcal S_R.
\end{cases}
\]
Defining
\[
\alpha:=|\mathcal S_R|\bar\phi\in(0,1),
\]
the bound $\lambda_i^s\geq\bar\phi$ implies that the previously
defined $T$-step kernel $K_T$ satisfies
\[
K_T(s,\cdot)\geq\alpha\nu(\cdot),
\qquad \forall s\in\mathcal S.
\]
This is a Doeblin minorization. Equivalently, each row of $K_T$
contains the same recurrent-state component of total weight $\alpha$.

\textbf{Step 6:} We can then subtract $\bar\phi$ from each
coefficient $\lambda_i^s$ to obtain:

\begin{align*}
\tilde{v}_{T}(s) = \gamma^T \sum_{i \in \mathcal{S}_R} (\lambda_i^s - \bar\phi) \tilde{v}_0(i) + \gamma^T \bar\phi \sum_{i \in \mathcal{S}_R}  \tilde{v}_0(i) +  \gamma^T\sum_{j \in \mathcal{S}_T} \lambda_j^s \tilde{v}_0(j)  -  \gamma^T \sum_k \tilde E_k \tilde{v}_0(k).
\end{align*}

Having this expression, we can derive upper and lower bounds on the entries of $\tilde{v}_{T}(s)$ in terms of the maximum and minimum entries of $ \tilde{v}_0(i)$. Combining the states back, we have:

\begin{align*}
\tilde{v}_{T}(s) \ge \gamma^T\left(  (1-\bar\phi |\mathcal{S}_R|) \min_i \tilde{v}_0(i) + \bar\phi \sum_{i \in \mathcal{S}_R}\tilde{v}_0(i) -  \sum_k \tilde E_k \tilde{v}_0(k) \right)    
\end{align*}

and

\begin{align*}
\tilde{v}_{T}(s) \le \gamma^T\left(  (1-\bar\phi |\mathcal{S}_R|) \max_i \tilde{v}_0(i) + \bar\phi \sum_{i \in \mathcal{S}_R}\tilde{v}_0(i) - \sum_k \tilde E_k \tilde{v}_0(k) \right).   
\end{align*}

Subtracting one from another gives the bound on the span of $\tilde{v}_{T}$:
\begin{align*}
\spanv{\tilde{v}_{T}} \le  \gamma^T (1-\bar\phi |\mathcal{S}_R|) \spanv{\tilde{v}_0}
\end{align*}
Denoting $\iota:= 1-\bar\phi |\mathcal{S}_R|$ completes the proof of the theorem.
\end{proof}

We now proceed to the proofs of both Corollaries, which start with the following lemma:

\begin{lemma} \label{lem:rew_by_span}
For any SAP $z \in \pi_t$, where $\pi_t$ is a policy implied by values $v_t$
\[ r^z \ge - \gamma\frac{ \spanv{v_t}}{C}.\]
\end{lemma}

\begin{proof}
Denote the SAP chosen by the optimal policy in state $s$ by $z^*$. Since SAP $z$ is a greedy choice with respect to $v_t$, its advantage with respect to this value vector is larger, which yields:
\begin{align*}
0 &\leq  \adv{z}{v_t} - \adv{z^*}{v_t} = r^z - r^{z^*} + \sum_i (c^z_i- c^{z^*}_i) v_t(i) \\
&\le r^z +  \gamma \frac{\spanv{v_t}}{C},
\end{align*}
where the last inequality follows from the fact that optimal actions have $0$ rewards in a normalized MDP and the redistribution argument applied to values $v_t(i)$ and coefficients $(c^z_i- c^{z^*}_i)$.
\end{proof}

This lemma allows us to prove the corollaries.

\begin{proof}[Proof of Corollary~\ref{cor:iteration_complexity_disc}]
Run the algorithm for
$$
t=\frac{\log(1/\epsilon)+\log(1/(1-\gamma))+\log(\spanv{v_0}/C)}
{\log(1/\gamma) + \log(1/\iota)/n^2}
$$
iterations.
By Theorem~\ref{thm:sync_no_lr}, after this number of iterations the obtained
vector $v_t$ satisfies
$$
\frac{\spanv{v_t}}{C} \le \epsilon (1-\gamma).
$$
Applying Lemma~\ref{lem:rew_by_span}, we obtain that for all SAPs in $\pi_t$
their rewards are lower bounded by $-\gamma \epsilon (1-\gamma)$, which in turn
implies that the standard value function $V^{\pi_t}$ is lower bounded by
$\epsilon$.

Since normalization preserves differences in values between policies, the
same bound holds for the original MDP.
\end{proof}

The proof of the second corollary is almost identical to the first one.

\begin{proof}[Proof of Corollary~\ref{cor:iteration_complexity_av}]
Run the algorithm for
$$
t=\frac{\log(1/\epsilon)+\log(\spanv{v_0}/C)}
{\log(1/\iota)/n^2}
$$
iterations.
By Theorem~\ref{thm:sync_no_lr}, after this number of iterations the obtained
vector $v_t$ satisfies
$$
\frac{\spanv{v_t}}{C} \le \epsilon.
$$
Applying Lemma~\ref{lem:rew_by_span}, we obtain that for all SAPs in $\pi_t$
their rewards are lower bounded by $-\epsilon$, which in turn implies that
the gain of this policy $\rho^{\pi_t}(s)$ is lower bounded by $\epsilon$
for any state $s$.

Since normalization preserves differences in gains between policies (via
the preservation of value differences and Lemma~\ref{lemma:average equivalence}),
the same bound holds for the original MDP.
\end{proof}

\section{Additional Discussion}

\subsection{Algebraic and Geometric Intuition}
\label{app:intuition}

\paragraph{Algebraic motivation.}
When $\gamma=1$, the classical policy-evaluation matrix $I-P^\pi$ is singular because $P^\pi\mathbf{1}=\mathbf{1}$.
Writing $\tilde v^\pi=v^\pi/C$, our representation instead solves
\[
(I+\gamma E-\gamma P^\pi)\tilde v^\pi=R^\pi,
\qquad E=\mathbf{1}\mathbf{1}^{\top}.
\]
The additional term $\gamma E\tilde v^\pi$ is identical at every state and therefore leaves state differences unchanged while fixing the otherwise undetermined constant component.
In particular, at $\gamma=1$, multiplying by the stationary
distribution of a unichain policy gives $\mathbf{1}^{\top}\tilde v^\pi=\rho^\pi$, so the same equation becomes $(I-P^\pi)\tilde v^\pi=R^\pi-\rho^\pi\mathbf{1}$.
Thus, $\tilde v^\pi$ represents the bias with a specified additive constant.
This provides a common algebraic representation for both reward criteria without requiring geometric arguments.

\paragraph{Geometric interpretation.}
The policy values determine a hyperplane in SAP space.
Normalization makes the optimal-policy hyperplane horizontal: optimal SAPs have zero rewards, whereas non-optimal SAPs have negative rewards under the strict advantage gap.
Consequently, the normalized iterates describe deviations from the optimal values, and $\spanv{v_t}$ measures their variation across states.
Algebraically, this is an analysis of value errors; geometrically, it describes the flattening of the iterate hyperplanes toward the optimal hyperplane, up to a common value shift.

\paragraph{Why additional contraction occurs.}
The strict advantage gap guarantees that every stochastic update matrix contains a uniformly positive component of $P^\star$.
Unichain structure and aperiodicity then ensure that, over a sufficiently long block, all rows of the update kernel share positive mass on the recurrent states.
Writing this Doeblin minorization as
\[
K_T=\alpha\mathbf{1}\nu^\top+(1-\alpha)Q,
\qquad \alpha\in(0,1),
\]
where $\nu$ is a probability vector and $Q$ is stochastic, gives
\[
\spanv{K_Tx}
=(1-\alpha)\spanv{Qx}
\le (1-\alpha)\spanv{x}.
\]
The common component cancels in the span, as do the state-independent shifts generated by $E$.
Together with the discount factor, this yields the block contraction $\gamma^T(1-\alpha)$, with $1-\alpha=\iota$.

\subsection{Conservativeness of the Bound}
\label{app:conservativeness of the results}

The bound in Theorem~4.2 is intentionally conservative. Its purpose is to
establish the existence of an additional MDP-dependent geometric contraction,
rather than to provide a tight prediction of the empirical convergence rate.

There are several sources of conservativeness. First, the proof uses the
minimum nonzero transition probability $P_{\min}^\star$ together with the
uniform block length $T=n^2$. This effectively tracks only a guaranteed
minimal amount of probability mass that can propagate through the recurrent
class over $T$ steps. In a given MDP, the optimal-policy-induced chain may mix
much faster, so substantially shorter blocks and larger effective weights are
possible.

Second, we lower bound the diagonal entries of $D_t$ using the initial span
$\spanv{\tilde v_0}$. In the actual dynamics, the relevant quantity depends on
$\spanv{\tilde v_t}$. As the span contracts, these diagonal weights increase,
so the contribution of $P^\star$ becomes progressively larger. In fact, the
entries $D_t[i,i]$ approach $1$ as the iterates converge, which means that the
update becomes increasingly dominated by the optimal-policy dynamics. The
proof does not exploit this improvement and instead uses a single uniform
lower bound valid for all iterations.

Third, the argument isolates only a small common portion of probability mass,
namely the component of size $\bar\phi$ that is guaranteed to be present for
every state. The remaining probability mass is treated only through worst-case
bounds. In practice, this remaining mass can also be redistributed in a way
that further reduces the span, but this additional contraction is not captured
by the theorem. Thus, the factor
\[
\iota=1-\bar\phi|\mathcal S_R|
\]
should be interpreted as a worst-case contraction certificate rather than as
an exact convergence rate.

Consequently, we do not claim that $\iota$ is tight or uniformly bounded away
from $1$ over the class of MDPs satisfying Assumption~4.1. The empirical rate
can be substantially faster than the theoretical bound.

Figure~\ref{fig:VI_conv_example_theory} illustrates this point. It reproduces
the experiments from Figure~\ref{fig:VI_conv_example} and additionally overlays
the theoretical bounds from Theorem~4.2. The observed convergence is
considerably faster, consistent with the conservative nature of the estimates
used in the proof.

\begin{figure}[h]
    \centering
    \includegraphics[width=\linewidth]{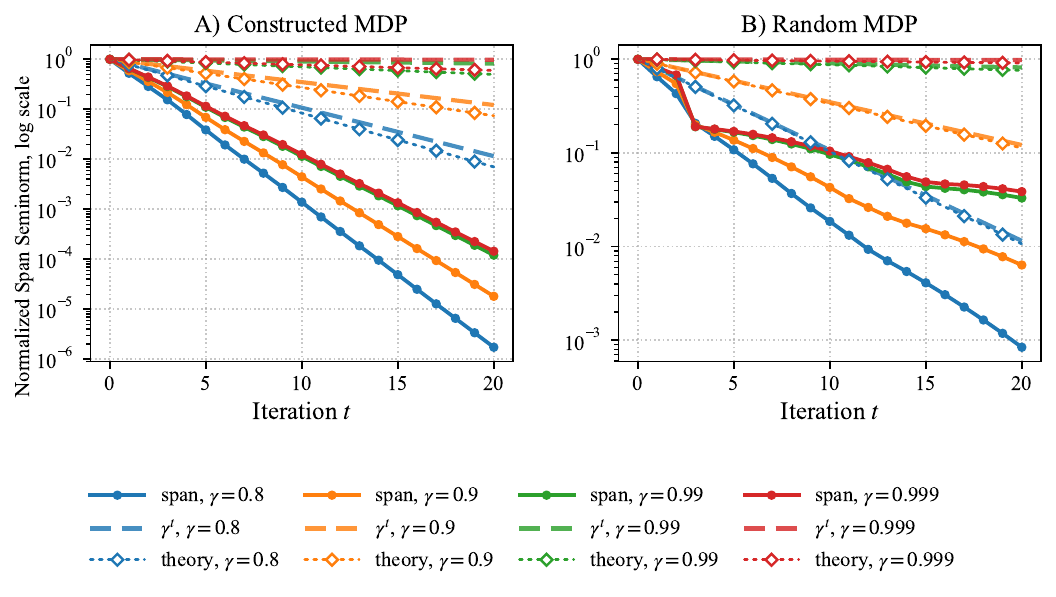}
    \caption{The same experiments as in Figure~\ref{fig:VI_conv_example}, with
    the theoretical convergence bounds from Theorem~4.2 added. The empirical
    convergence is substantially faster than the worst-case theoretical bound.}
    \label{fig:VI_conv_example_theory}
\end{figure}

\clearpage

\section*{NeurIPS Paper Checklist}

\begin{enumerate}

\item {\bf Claims}
    \item[] Question: Do the main claims made in the abstract and introduction accurately reflect the paper's contributions and scope?
    \item[] Answer: \answerYes{} 
    \item[] Justification: We believe that the claims made in the abstract and introduction accurately reflect the paper's contributions and scope.

\item {\bf Limitations}
    \item[] Question: Does the paper discuss the limitations of the work performed by the authors?
    \item[] Answer: \answerYes{} 
    \item[] Justification: We discuss limitations together with a validity of the assumptions in Section~\ref{sec:discussion}. 

\item {\bf Theory assumptions and proofs}
    \item[] Question: For each theoretical result, does the paper provide the full set of assumptions and a complete (and correct) proof?
    \item[] Answer: \answerYes{} 
    \item[] Justification: Yes, the paper provides the full set of assumptions and proofs for the claimed theoretical results.

    \item {\bf Experimental result reproducibility}
    \item[] Question: Does the paper fully disclose all the information needed to reproduce the main experimental results of the paper to the extent that it affects the main claims and/or conclusions of the paper (regardless of whether the code and data are provided or not)?
    \item[] Answer: \answerNA{} 
    \item[] Justification: 

\item {\bf Open access to data and code}
    \item[] Question: Does the paper provide open access to the data and code, with sufficient instructions to faithfully reproduce the main experimental results, as described in supplemental material?
    \item[] Answer: \answerNA{} 
    \item[] Justification:

\item {\bf Experimental setting/details}
    \item[] Question: Does the paper specify all the training and test details (e.g., data splits, hyperparameters, how they were chosen, type of optimizer) necessary to understand the results?
    \item[] Answer:  \answerNA{} 
    \item[] Justification: 

\item {\bf Experiment statistical significance}
    \item[] Question: Does the paper report error bars suitably and correctly defined or other appropriate information about the statistical significance of the experiments?
    \item[] Answer:  \answerNA{} 
    \item[] Justification: 

\item {\bf Experiments compute resources}
    \item[] Question: For each experiment, does the paper provide sufficient information on the computer resources (type of compute workers, memory, time of execution) needed to reproduce the experiments?
    \item[] Answer:  \answerNA{} 
    \item[] Justification: 
    
\item {\bf Code of ethics}
    \item[] Question: Does the research conducted in the paper conform, in every respect, with the NeurIPS Code of Ethics \url{https://neurips.cc/public/EthicsGuidelines}?
    \item[] Answer: \answerYes 
    \item[] Justification:

\item {\bf Broader impacts}
    \item[] Question: Does the paper discuss both potential positive societal impacts and negative societal impacts of the work performed?
    \item[] Answer:  \answerNA{} 
    \item[] Justification: This is purely theoretical paper and we believe there is no societal impact.
    
\item {\bf Safeguards}
    \item[] Question: Does the paper describe safeguards that have been put in place for responsible release of data or models that have a high risk for misuse (e.g., pre-trained language models, image generators, or scraped datasets)?
    \item[] Answer: \answerNA{} 
    \item[] Justification:

\item {\bf Licenses for existing assets}
    \item[] Question: Are the creators or original owners of assets (e.g., code, data, models), used in the paper, properly credited and are the license and terms of use explicitly mentioned and properly respected?
    \item[] Answer: \answerNA{} 
    \item[] Justification: 

\item {\bf New assets}
    \item[] Question: Are new assets introduced in the paper well documented and is the documentation provided alongside the assets?
    \item[] Answer: \answerNA{} 
    \item[] Justification: 

\item {\bf Crowdsourcing and research with human subjects}
    \item[] Question: For crowdsourcing experiments and research with human subjects, does the paper include the full text of instructions given to participants and screenshots, if applicable, as well as details about compensation (if any)? 
    \item[] Answer: \answerNA{} 
    \item[] Justification: 

\item {\bf Institutional review board (IRB) approvals or equivalent for research with human subjects}
    \item[] Question: Does the paper describe potential risks incurred by study participants, whether such risks were disclosed to the subjects, and whether Institutional Review Board (IRB) approvals (or an equivalent approval/review based on the requirements of your country or institution) were obtained?
    \item[] Answer: \answerNA{} 
    \item[] Justification: 

\item {\bf Declaration of LLM usage}
    \item[] Question: Does the paper describe the usage of LLMs if it is an important, original, or non-standard component of the core methods in this research? Note that if the LLM is used only for writing, editing, or formatting purposes and does \emph{not} impact the core methodology, scientific rigor, or originality of the research, declaration is not required.
    \item[] Answer: \answerNA{} 
    \item[] Justification:

\end{enumerate}

\end{document}